\documentclass{article}
\usepackage{graphicx}
\usepackage{subcaption}
\usepackage{amsmath}
\usepackage{PRIMEarxiv}
\usepackage{amsthm}      
\usepackage{amssymb}
\usepackage[utf8]{inputenc} 
\usepackage[T1]{fontenc}    
\usepackage{hyperref}       
\usepackage{url}            
\usepackage{booktabs}       
\usepackage{amsfonts}       
\usepackage{nicefrac}       
\usepackage{microtype}      
\usepackage{lipsum}
\usepackage{fancyhdr}       
\usepackage{graphicx}       
\graphicspath{{media/}}     
\usepackage{algorithm}   
\usepackage{algorithmic} 
\newtheorem{theorem}{Theorem}
\newtheorem{lemma}{Lemma}
\newtheorem{proposition}{Proposition}

\newtheorem{assumption}{Assumption}
\usepackage{float}       
\title{Tail-Safe Hedging: Explainable Risk-Sensitive Reinforcement Learning with a White-Box CBF--QP Safety Layer in Arbitrage-Free Markets}

\author{
  ZhangJian'an \\
  Guanghua School of Management, Peking University \\
  Peking University \\
  Beijing, China\\
  \texttt{2501111059@stu.pku.edu.cn}
}

\begin{document}
\pagestyle{plain}   
\maketitle

\begin{abstract}
We introduce \textbf{Tail-Safe}, a deployability-oriented framework for derivatives hedging that unifies \emph{distributional, risk-sensitive reinforcement learning} with a \emph{white-box} safety layer tailored to financial constraints. 
The learning module combines an IQN-based distributional critic with a CVaR objective (IQN--CVaR--PPO), and is stabilized by a \emph{Tail-Coverage Controller} that regulates quantile sampling via temperature tilting and tail-boosting to preserve effective tail mass at small $\alpha$. 
The safety module solves a convex \emph{CBF--QP} at every step, enforcing discrete-time control barrier inequalities together with a finance-specific constraint set (ellipsoidal no-trade band, box/rate limits, and a sign-consistency gate). 
Crucially, the safety layer is \emph{explainable}: it exposes active sets, tightness, rate utilization, gate scores, slack, and solver status---a self-contained audit trail aligned with model-risk governance.

We provide formal guarantees: (i) \emph{robust forward invariance} of the safe set under bounded model mismatch; (ii) a \emph{minimal-deviation projection} interpretation of the QP; (iii) a \emph{KL--DRO} upper bound showing that per-state KL regularization controls worst-case CVaR; (iv) \emph{concentration} and sample-complexity of the temperature-tilted CVaR estimator with explicit coverage-mismatch terms; (v) a \emph{CVaR trust-region} improvement inequality under KL-limited updates; (vi) \emph{feasibility persistence} via expiry-aware NTB shrinkage and rate tightening; and (vii) \emph{negative-advantage suppression} induced by the sign-consistency gate.
Empirically, in arbitrage-free, microstructure-aware synthetic markets (SSVI$\!\rightarrow$Dupire$\!\rightarrow$VIX with ABIDES/MockLOB execution), Tail-Safe improves left-tail risk while preserving central performance and yields zero hard-constraint violations whenever the QP is feasible with zero slack. 
We also map telemetry to governance workflows (dashboards, triggers, and incident taxonomy) to support auditability.
Limitations include the use of synthetic environments, simplified execution, and the absence of real-data replay; these are deliberate to isolate methodological contributions and are actionable in future work.
\end{abstract}


\section{Introduction}

Modern AI agents for derivatives hedging and portfolio risk management must remain \emph{tail-safe}, \emph{robust under distribution shift}, and \emph{auditable} for model risk governance and regulatory review. In equity–volatility books such as SPX–VIX, a small number of rare but consequential events (flash crashes, volatility spikes, liquidity droughts) dominate economic capital, PnL attribution, and operational risk. Agents optimized purely for risk-neutral or average performance---e.g., standard reinforcement learning (RL) objectives or black-box deep hedgers---can look competitive ex post yet still incur catastrophic left-tail losses, violate hard business rules (leverage, short-sale limits, liquidity/min-trade-size, turnover and drawdown constraints), or fail under microstructure stress~\cite{Buehler2019DeepHedging,Hambly2023RLFinance,Sun2023RLTrading,Pippas2025CSUR}. These shortcomings are amplified when training is performed in stylized simulators that under-represent jump risk, regime changes, or execution frictions. At the same time, realistic deployment requires reasoning about market impact, order-book dynamics, and latency that classical optimal execution models capture only partially~\cite{AlmgrenChriss2001,ObizhaevaWang2013,Kyle1985}. Taken together, these pressures motivate methods that (i) \emph{optimize tail risk} rather than only the mean, (ii) \emph{guarantee state-wise safety by construction} at execution time, and (iii) \emph{explain interventions} through telemetry that supports model validation and external audit.

\paragraph{From risk-neutral RL to risk-sensitive, distributional RL.}
A central ingredient is to reason about \emph{distributions of returns} rather than their expectations. Distributional RL explicitly models the random return $Z^\pi$ and its quantiles, enabling direct access to tail statistics and risk measures (e.g., VaR/CVaR) that are meaningful for trading books~\cite{Bellemare2017Distributional,Dabney2018IQN}. On-policy trust-region style updates (TRPO/PPO) stabilize training through KL control and clipping, reducing policy oscillations that can otherwise translate into erratic trading~\cite{Schulman2015TRPO,Schulman2017PPO}. Building on these foundations, risk-sensitive RL incorporates coherent risk measures and chance constraints into the objective or constraints of the Markov decision process, yielding actor–critic and policy-gradient algorithms with convergence guarantees~\cite{Chow2014CVaRMDP,Tamar2015CoherentRisk,Chow2018VaRCVaR}. Recent advances scale CVaR-style training, improve tail estimators, and study policy updates under distribution shift and partial model misspecification~\cite{Lu2024TVDRL,SundharRamesh2024DRMBRL,Shi2024JMLR}. These directions are especially pertinent in finance, where left-tail losses dominate risk budgets, and where audits require transparent objectives and measurable risk appetites~\cite{KellyXiu2023FinML,Sun2023RLTrading,Hambly2023RLFinance,Pippas2025CSUR}. In this work we adopt distributional critics and quantile-based losses together with explicit CVaR optimization to shape the tail, while retaining PPO-style update discipline to keep policy drift controlled during training.

\paragraph{Safety by construction via control barrier functions (CBFs).}
Even with risk-sensitive training, \emph{state-wise} safety constraints (e.g., leverage $\leq$ cap, inventory/rate boxes, NTB and sign consistency) must be satisfied at each decision, not only on average. Typical safe RL approaches enforce constraints asymptotically, in expectation, or via learned proxies, leaving residual violation risk during exploration or regime shifts~\cite{Ray2019SafetyGym,Ji2023SafetyGymnasium,Zhao2023StatewiseSafeRL}. CBF-based safety layers provide a complementary, \emph{white-box} mechanism: a control barrier function $h(x)$ defines a forward-invariant safe set $\mathcal{S}=\{x:h(x)\ge 0\}$; at each step a convex quadratic program (QP) minimally modifies a nominal action $u_0$ to ensure a discrete-time CBF condition such as
\[
h(x_{t+1}) - (1-\alpha) h(x_t) \;\ge\; 0,
\]
subject to box/rate and market constraints, where $\alpha\in(0,1]$ regulates conservatism~\cite{Ames2017CBF_TAC,Ames2019CBF_Tutorial}. The CBF literature now includes high-order and robust variants, differentiable formulations for learning, discrete-time treatments, and observer/disturbance-aware designs suitable for implementation~\cite{Garg2024CBFAdvances,Ma2022DiffCBF,Yang2022DiffCBF,Cortez2022DPC,Choi2021CBVF,Jankovic2018RobustCBF,Xu2015RobustCBF,WangXu2024DOB_CBF}. For financial RL, a CBF--QP safety layer is attractive because it (i) encodes leverage, liquidity, and sign-consistency gates \emph{by construction}; (ii) exposes KKT multipliers, active sets, and slack, yielding human-interpretable \emph{reasons} for action overrides; and (iii) composes with arbitrary nominal policies without retraining.

\paragraph{Robustness to model misspecification via distributional robustness and KL control.}
Agents trained on synthetic markets inevitably face deployment regimes with different volatility-of-volatility, jump intensities, liquidity states, and order-flow autocorrelations. Distributionally robust MDPs (DRMDPs) guard against such misspecification by optimizing worst-case value within ambiguity sets defined by $f$-divergence or Wasserstein balls; practical algorithms cover online/offline and model-based/model-free settings~\cite{Shi2024JMLR,Lu2024TVDRL,SundharRamesh2024DRMBRL,Liu2025DRMDP}. In parallel, KL-regularized policy updates---as in TRPO/PPO---admit a robust-optimization reading: a per-state KL penalty is the Fenchel dual of an adversary constrained to a KL ball around the behavior distribution, providing an interpretable \emph{risk-budget knob} for conservatism during tail-focused learning~\cite{Schulman2015TRPO,Schulman2017PPO}. We further view adversarial/stress-testing methods from robust RL as complementary tools to expose fragilities before live trading~\cite{Pinto2017RARL,Kumar2020CQL}.

\paragraph{Finance-grounded synthetic evaluation.}
Auditable studies require \emph{arbitrage-free} yet \emph{stressable} simulators. Practitioners routinely calibrate SSVI volatility surfaces under static no-arbitrage constraints~\cite{GatheralJacquier2014} and derive state-dependent dynamics via Dupire local volatility~\cite{Dupire1994}. VIX-consistent legs can be inferred from the surface, enabling joint SPX–VIX exposures. Execution realism demands microstructure-aware simulators and impact models: Almgren--Chriss captures linear temporary impact and risk/cost trade-offs, while Obizhaeva--Wang models transient resilience~\cite{AlmgrenChriss2001,ObizhaevaWang2013}. Open-source agent-based platforms such as ABIDES and ABIDES-Gym bring limit-order books, latencies, and interacting agents to RL research, bridging the gap between stylized price processes and exchange-like environments~\cite{Byrd2020ABIDES,Byrd2019ABIDES,ABIDESGym2021}. This stack---SSVI$\!\to$Dupire$\!\to$VIX with LOB execution and impact---has become common in surveys and practice for reproducible stress testing and policy benchmarking~\cite{Hambly2023RLFinance,Sun2023RLTrading,Pippas2025CSUR}.

\paragraph{This paper: Tail-Safe hedging with an explainable safety layer.}
We propose \textbf{Tail-Safe}, a hybrid \emph{learn--then--filter} framework for SPX–VIX hedging that combines risk-sensitive learning with white-box safety and robustness. (i) The learning component trains a distributional critic with CVaR objectives, stabilized by a \emph{quantile-coverage controller} that tracks the effective tail mass during training to reduce estimator variance and policy churn. (ii) The execution component enforces \emph{state-wise} constraints via a finance-specific CBF--QP layer that implements NTB shrinkage, box/rate limits, and sign-consistency, while exporting telemetry (active-set identity, tightest constraint, slack/rate utilization) suitable for audit. Policy updates follow KL-regularized PPO with an EMA reference policy to cap per-epoch drift~\cite{Schulman2017PPO}. Evaluation occurs in an arbitrage-free SSVI$\!\rightarrow$Dupire$\!\rightarrow$VIX pipeline with ABIDES/LOB execution and classical impact models~\cite{GatheralJacquier2014,Dupire1994,Byrd2020ABIDES,AlmgrenChriss2001,ObizhaevaWang2013}.

\paragraph{Mathematical guarantees and explainability (paper highlights).}
To meet top-tier standards and strengthen interpretability, we develop and prove:
\begin{itemize}
  \item \textbf{Discrete-time invariance \& minimal-deviation (Theorem~1).} For the proposed discrete-time CBF constraints, if the QP is feasible with zero slack, the safe set is forward-invariant. The $H$-metric objective selects the uniquely closest feasible action to the nominal proposal, making the safety intervention \emph{predictable} and \emph{stable}~\cite{Ames2017CBF_TAC,Garg2024CBFAdvances}.
  \item \textbf{DRO--KL equivalence (Theorem~2).} A per-state KL penalty in PPO corresponds to the Fenchel dual of a KL-ball worst-case objective, quantifying conservatism and providing a tunable \emph{risk budget} that links trust-region size to adversarial uncertainty~\cite{Lu2024TVDRL,SundharRamesh2024DRMBRL,Schulman2015TRPO}.
  \item \textbf{Tail-coverage stabilization (Theorem~3).} A PID-style quantile-coverage controller stabilizes the Monte Carlo CVaR estimator by regulating effective tail mass around a scheduled $\alpha$, yielding finite-sample variance control and bounding the gap between scheduled and realized tail coverage under mild regularity.
  \item \textbf{Feasibility under tail guards (Proposition~1).} With expiry-aware NTB shrinkage, sign-consistency gates, and realistic microstructure and inventory envelopes (rate/box limits), the safety QP admits a nonempty feasible set and avoids deadlock in stressed conditions.
  \item \textbf{Telemetry--KKT correspondence (Lemma~1).} Tightness indicators and dual variables map one-to-one to business-rule interventions (leverage/short-sale/rate caps), producing auditable, human-readable explanations of \emph{why} and \emph{by how much} a proposed trade was modified.
\end{itemize}

\paragraph{Contributions.}
\textbf{(1) Tail-risk learning with stability.} We integrate distributional RL (IQN) and CVaR objectives with a \emph{quantile-coverage controller} and PPO+KL updates, improving left-tail behavior without large policy drift~\cite{Bellemare2017Distributional,Dabney2018IQN,Schulman2017PPO}.\\
\textbf{(2) White-box CBF--QP safety.} We design a finance-specific CBF--QP that enforces NTB, box/rate, and sign-consistency gates with telemetry, enabling \emph{explainable} state-wise safety consistent with the control literature~\cite{Ames2017CBF_TAC,Garg2024CBFAdvances,Zhao2023StatewiseSafeRL}.\\
\textbf{(3) Robustness to shift.} We couple KL-regularized updates with DRO-motivated penalties to guard against simulator misspecification and out-of-distribution stress~\cite{Lu2024TVDRL,SundharRamesh2024DRMBRL,Shi2024JMLR}.\\
\textbf{(4) Finance-grounded evaluation.} We build an arbitrage-free SSVI/Dupire/VIX simulator with ABIDES execution and impact models to enable reproducible, stressable studies aligned with practitioner workflows~\cite{GatheralJacquier2014,Dupire1994,Byrd2020ABIDES,AlmgrenChriss2001,ObizhaevaWang2013}.

\paragraph{Positioning and outlook.}
Compared with black-box deep hedgers or unconstrained RL, \textbf{Tail-Safe} delivers: (i) \emph{hard safety} by construction with audit-ready telemetry; (ii) \emph{tail shaping} via stabilized CVaR learning; and (iii) \emph{robustness} through DRO/KL regularization---all within markets that respect no-arbitrage and microstructure. This responds to recent surveys calling for safe, trustworthy, and explainable financial RL~\cite{Hambly2023RLFinance,Sun2023RLTrading,Pippas2025CSUR}. Beyond SPX–VIX, the framework extends to multi-asset books with cross-gamma and funding constraints, richer impact/latency models, and formal certificates for state-dependent CBF margins and DRMDPs~\cite{Garg2024CBFAdvances,Liu2025DRMDP}. We view these directions as necessary steps toward deployable, regulator-ready AI agents for real-world hedging.

\section{Preliminaries \& Problem Setting}
\label{sec:prelim}

We consider a synthetic yet finance-grounded evaluation stack for equity--volatility hedging (e.g., SPX--VIX) that is \emph{arbitrage-free}, \emph{microstructure-aware}, and \emph{stressable}. 
The market generator follows a calibrated no-arbitrage implied-volatility pipeline \textbf{SSVI}$\!\rightarrow$\textbf{Dupire}$\!\rightarrow$\textbf{VIX}, while order execution is simulated through an agent-based limit-order-book (LOB) environment (ABIDES/MockLOB) with temporary and transient impact. 
This section formalizes the environment, the agent interface (states/actions/inventory), loss and tail-risk metrics, and the occupancy and KL notions used later.

\subsection{Arbitrage-Free Volatility Surfaces via SSVI}
\label{sec:ssvi}

For each maturity $\tau>0$, denote by $w(k,\tau)$ the total implied variance at log-moneyness $k=\log(K/F_\tau)$. 
We parameterize $w$ using the \emph{SSVI} family~\cite{GatheralJacquier2014}:
\begin{equation}
\label{eq:ssvi}
w(k,\tau) \;=\; \frac{\theta(\tau)}{2}\left(
1 \;+\; \rho(\tau)\,\varphi(\tau)\,k \;+\;
\sqrt{ \big(\varphi(\tau)\,k + \rho(\tau)\big)^2 + 1 - \rho(\tau)^2 }
\right),
\end{equation}
where $\theta(\tau)>0$ is the ATM total variance term-structure, $\varphi(\tau)>0$ the shape, and $\rho(\tau)\!\in\!(-1,1)$ the skew parameter. 
SSVI admits tractable \emph{static no-arbitrage} conditions (no butterfly/calendar arbitrage) as simple inequalities on $(\theta,\varphi,\rho)$ across maturities, yielding a calibrated surface free of static arbitrage~\cite{GatheralJacquier2014}. 
We fit $(\theta,\varphi,\rho)$ on SPX option data slices or stylized templates and use~\eqref{eq:ssvi} as the \emph{sole} input to downstream local-volatility and variance measures.

\subsection{Local Volatility via Dupire}
\label{sec:dupire}

Let $C(t,K)$ be the time-$t$ undiscounted call price with strike $K$ implied by the SSVI surface. 
Under mild regularity and risk-neutrality, the \emph{Dupire} local variance $\sigma_{\mathrm{loc}}^2(t,K)$ satisfies~\cite{Dupire1994}:
\begin{equation}
\label{eq:dupire}
\sigma_{\mathrm{loc}}^2(t,K)
\;=\;
\frac{\partial_t C(t,K)}{\tfrac{1}{2}K^2\,\partial_{KK} C(t,K)}\,,
\qquad
\partial_{KK} C(t,K) > 0.
\end{equation}
We numerically evaluate the derivatives of $C$ from the SSVI-implied smiles and generate price paths by simulating the local-volatility diffusion $dS_t = r_t S_t\,dt + \sigma_{\mathrm{loc}}(t,S_t) S_t\,dW_t$, where $r_t$ is the (possibly term-structured) risk-free rate. 
This produces an arbitrage-consistent equity leg with skews and term-structure inherited from SSVI.

\subsection{VIX Leg from Surface-Consistent Variance}
\label{sec:vix}

Consistent with variance-swap replication~\cite{Demeterfi1999VolSwap}, the \emph{30-day VIX} at time $t$ can be computed from out-of-the-money option prices on the surface as (continuous limit, simplified form; see Cboe white paper~\cite{CBOEWhitePaper2019}):
\begin{equation}
\label{eq:vix}
\mathrm{VIX}^2(t)
\;=\;
\frac{2}{T}\,e^{rT}
\left(
\int_{0}^{F} \frac{P(K,T)}{K^2}\,dK
\;+\;
\int_{F}^{\infty} \frac{C(K,T)}{K^2}\,dK
\right),
\qquad T=\text{30D},
\end{equation}
where $F$ is the forward, and $P,C$ are OTM put/call prices derived from the SSVI surface at horizon $T$. 
We thus construct a \emph{surface-consistent} VIX leg---either as a tradable future proxy or as options on VIX—ensuring joint equity–volatility dynamics coherent with the SSVI/Dupire world.

\subsection{Execution, Microstructure, and Impact}
\label{sec:execution}

We embed the above market into an agent-based discrete-event LOB simulator using ABIDES/ABIDES-Gym~\cite{Byrd2020ABIDES,Amrouni2021ABIDESGym}, which exposes realistic order matching, latency, and interacting background agents. 
Our execution price model follows classical impact literature~\cite{AlmgrenChriss2001,ObizhaevaWang2013,Gatheral2010NoDynArb}:
\begin{align}
\label{eq:impact}
p_t^{\mathrm{exec}}
&= m_t \;+\; \frac{s_t}{2}\,\mathrm{sgn}(u_t) \;+\; \eta\,u_t \;+\; \sum_{j=0}^{\infty} G(j)\,u_{t-j},\\
q_{t+1} &= q_t + u_t, \qquad \Pi_{t+1}-\Pi_t \;=\; q_t\,(m_{t+1}-m_t) \;-\; \mathrm{Cost}(u_t),
\end{align}
where $m_t$ is the mid-price, $s_t$ the spread, $u_t$ the signed trade size, $\eta$ the \emph{temporary} (linear) impact coefficient (Almgren--Chriss), and $G(\cdot)$ a \emph{transient} resilience kernel (Obizhaeva--Wang). 
The realized trading cost $\mathrm{Cost}(u_t)$ aggregates spread, temporary impact, and transient slippage implied by the ABIDES fill process. 
We adopt \emph{no-dynamic-arbitrage} constraints on impact to avoid price manipulation~\cite{Gatheral2010NoDynArb}.

\paragraph{Stress dimensions.}
For systematic OOD testing, we perturb SSVI parameters $(\theta,\varphi,\rho)$ (level/slope/curvature), equity–vol correlation, and impact strength/decay, as well as time-to-expiry (near-expiry regimes). 
All stress configurations remain free of \emph{static} arbitrage by construction (\S\ref{sec:ssvi}).

\subsection{Task, Agent Interface, and Loss}
\label{sec:task}

We consider a finite-horizon, discrete-time hedging task $t=0,1,\dots,T{-}1$. 
At time $t$, the agent observes a state $x_t\in\mathbb{R}^d$ (prices/Greeks/surface features/time-to-expiry/inventory/LOB and execution features) and issues an action $u_t\in\mathbb{R}^m$ (trade vector across SPX/VIX instruments). 
Inventory $q_t$ evolves per~\eqref{eq:impact}. 
Let $\Pi_T(\pi;\omega)$ denote terminal P\&L under policy $\pi$ on path $\omega$; we define \emph{loss} as $L_T(\pi;\omega) = -\,\Pi_T(\pi;\omega)$, including transaction costs and slippage from the execution adapter.

\subsection{Risk Measures: VaR and ES/CVaR}
\label{sec:risk}

For a confidence level $\alpha\in(0,1)$ and random loss $L$, define
\begin{equation}
\operatorname{VaR}_\alpha(L) \;=\; \inf\{z\in\mathbb{R}:\;\mathbb{P}(L\le z)\ge \alpha\}, 
\qquad
\operatorname{ES}_\alpha(L) \;=\; \mathbb{E}\!\left[L \,\middle|\, L\ge \operatorname{VaR}_\alpha(L)\right],
\end{equation}
i.e., \emph{Expected Shortfall} (a.k.a.\ CVaR) is the conditional mean of losses beyond VaR. 
ES is a coherent risk measure~\cite{AcerbiTasche2002} and admits convex optimization surrogates~\cite{RockafellarUryasev2000}. 
We will evaluate policies using absolute-loss $\operatorname{VaR}_\alpha$ and $\operatorname{ES}_\alpha$ together with central-performance ratios.

\subsection{Occupancy Measures and Policy KL}
\label{sec:occ_kl}

Let $\pi(u|x)$ denote the stochastic policy, and let $d_\pi(x)$ be the \emph{state occupancy} over the finite horizon: 
$d_\pi(x) = \frac{1}{T}\sum_{t=0}^{T-1}\mathbb{P}_\pi(x_t=x)$. 
We also use the \emph{discounted} occupancy when appropriate.
For conservatism, per-state \emph{Kullback--Leibler} (KL) regularization between the updated policy and a reference $\pi_{\mathrm{ref}}$ is defined as
\begin{equation}
\label{eq:kl}
\mathrm{KL}\!\big(\pi(\cdot|x)\,\|\,\pi_{\mathrm{ref}}(\cdot|x)\big)
\;=\; \int \pi(u|x)\,\log\frac{\pi(u|x)}{\pi_{\mathrm{ref}}(u|x)}\,du,
\end{equation}
and aggregated across $x\sim d_\pi$.
KL plays a dual role: (i) as a \emph{trust-region} step-size proxy in policy optimization and (ii) as a distributionally robust regularizer (Fenchel dual of a KL-ball worst-case objective), thus controlling sensitivity to simulator misspecification~\cite{KullbackLeibler1951,RockafellarUryasev2000}.

\subsection{Notation Summary}
\label{sec:notation}

\begin{table}[H]
\centering
\caption{Notation used throughout the paper (symbols refer to their values at time $t$ unless otherwise noted).}
\label{tab:notation}
\vspace{0.3em}
\begin{tabular}{ll}
\toprule
Symbol & Meaning \\
\midrule
$S_t$ & Underlying equity mid-price; $K$ strike; $F$ forward \\
$w(k,\tau)$ & Total implied variance at log-moneyness $k$ and maturity $\tau$ (SSVI; Eq.~\eqref{eq:ssvi}) \\
$\sigma_{\mathrm{loc}}(t,S)$ & Dupire local volatility (Eq.~\eqref{eq:dupire}) \\
$\mathrm{VIX}(t)$ & 30D variance index computed from OTM option integrals (Eq.~\eqref{eq:vix}) \\
$m_t, s_t$ & LOB mid-price and bid--ask spread \\
$u_t, q_t$ & Trade vector (signed); inventory vector \\
$p_t^{\mathrm{exec}}$ & Execution price with spread/temporary/transient impact (Eq.~\eqref{eq:impact}) \\
$\Pi_T$ & Terminal P\&L; $L_T=-\Pi_T$ loss \\
$\operatorname{VaR}_\alpha, \operatorname{ES}_\alpha$ & Tail risk measures at confidence $\alpha$ (\S\ref{sec:risk}) \\
$\pi(u|x)$, $\pi_{\mathrm{ref}}(u|x)$ & Stochastic policy and EMA reference policy \\
$d_\pi(x)$ & (Discounted) state occupancy under $\pi$ (\S\ref{sec:occ_kl}) \\
$\mathrm{KL}(\pi\|\pi_{\mathrm{ref}})$ & Per-state Kullback--Leibler divergence (Eq.~\eqref{eq:kl}) \\
$r_t$ & Risk-free rate; $T$ hedging horizon length (steps); $\Delta t$ step size \\
\bottomrule
\end{tabular}
\end{table}

\paragraph{Remark (calibration and units).}
All option prices are undiscounted unless stated otherwise; $r_t$ denotes continuously compounded rates. 
We report P\&L in currency units and normalize risk (e.g., ES) by notional or premium where indicated for comparability across scenarios.

\section{Method: Tail-Safe Hedging Framework}
\label{sec:method}

This section presents \textbf{Tail-Safe}: a hybrid \emph{learn--then--filter} framework that couples risk-sensitive, distributional reinforcement learning with a \emph{white-box} CBF--QP safety layer. 
Section~\ref{sec:method:rl} details the IQN--CVaR--PPO learner and its KL-/entropy-regularized objective; 
Section~\ref{sec:method:coverage} introduces a \emph{Tail-Coverage Controller} that stabilizes low-$\alpha$ CVaR estimation via temperature-based quantile sampling and tail-boosting; 
Section~\ref{sec:method:cbf} specifies the discrete-time CBF constraints, the ellipsoidal no-trade band (NTB), box/rate limits, and a sign-consistency gate, together with audit-ready telemetry.
Figure~\ref{fig:framework} summarizes the overall architecture.

\begin{figure*}[t]
  \centering
  \includegraphics[width=\textwidth]{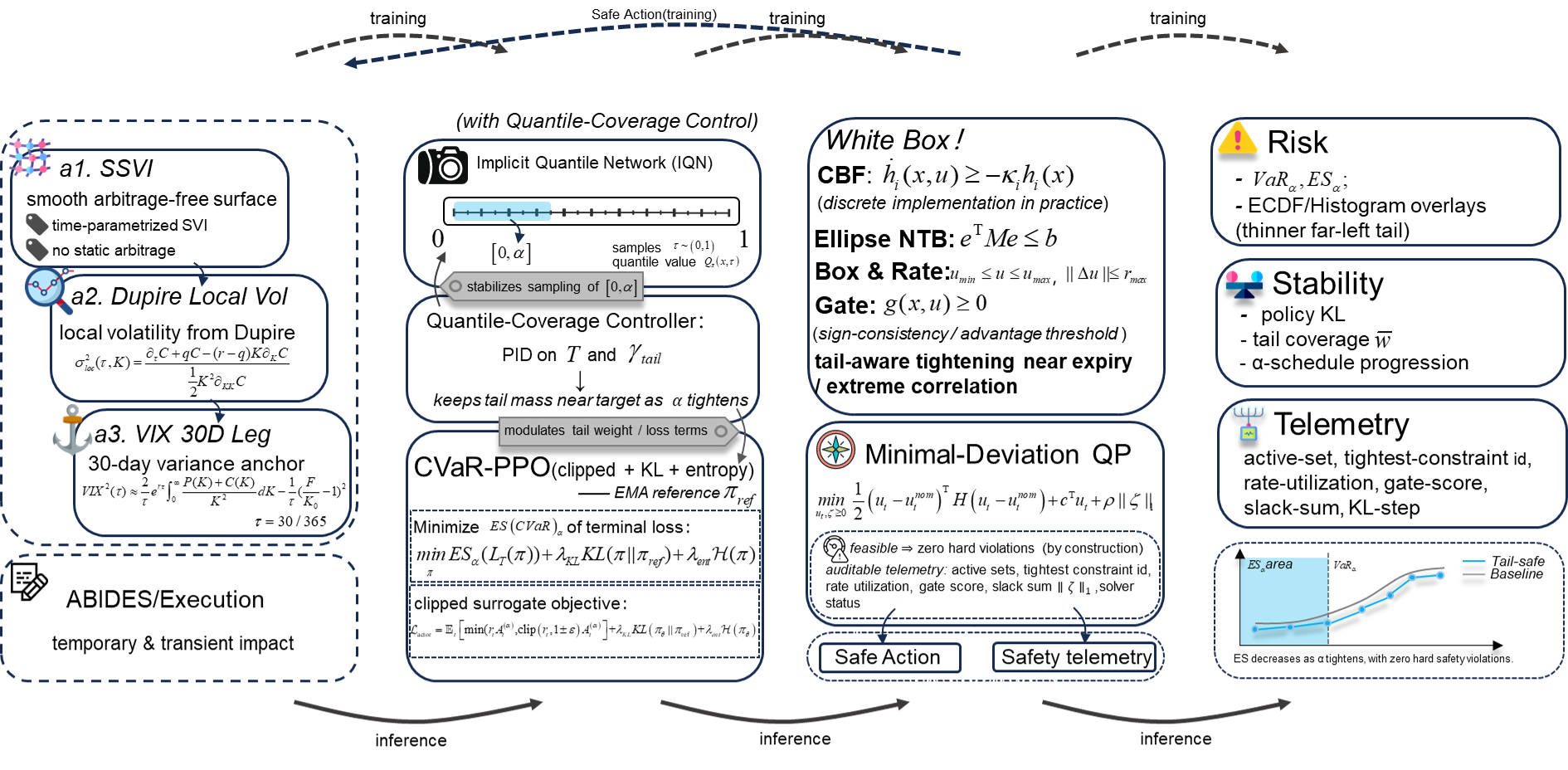}
  \caption{\textbf{Tail-Safe overview.} \textbf{(a) Market \& Execution}: SSVI $\rightarrow$ Dupire $\rightarrow$ 30D VIX; ABIDES/MockLOB execution with temporary and transient impact. 
  \textbf{(b) IQN--CVaR--PPO}: distributional critic via quantile regression, PPO with KL and entropy regularization. 
  \textbf{(c) White-box CBF--QP safety layer}: discrete-time CBF, ellipsoidal NTB, box/rate limits, and a sign-consistency gate solved as a convex QP. 
  \textbf{(d) Telemetry \& risk metrics}: $\mathrm{VaR}/\mathrm{ES}$, policy KL, tail coverage, active-set, tightest-constraint ID, rate utilization, gate score, and slack.}
  \label{fig:framework}
\end{figure*}

\subsection{Risk-Sensitive RL with IQN--CVaR--PPO}
\label{sec:method:rl}

\paragraph{Quantile networks and the CVaR objective.}
Distributional RL models the return distribution $Z_\pi(x,u)$ instead of its mean~\cite{Bellemare2017Distributional}. 
Implicit Quantile Networks (IQN) learn a differentiable quantile function $Q_\psi(x,u;\tau)\!\approx\!F^{-1}_{Z_\pi(\cdot|x,u)}(\tau)$~\cite{Dabney2018IQN}. 
Let loss be $L=-Z$. For $x$ fixed, the conditional CVaR admits the quantile integral
\begin{equation}
\label{eq:cvar-quantile}
\mathrm{CVaR}_\alpha(L\,|\,x)
\;=\;\frac{1}{\alpha}\int_{0}^{\alpha} F^{-1}_{L\,|\,x}(\tau)\,d\tau
\;\approx\;\frac{1}{K\alpha}\sum_{k=1}^{K}\mathbf{1}\{\tau_k\le\alpha\}\,\widehat{L}(x,\tau_k),
\end{equation}
with $\tau_k\!\sim\!\mathcal{U}(0,1)$ (or, in our case, a temperature-tilted distribution; cf.\ \S\ref{sec:method:coverage}) and $\widehat{L}(x,\tau)= -\,\mathbb{E}_{u\sim\pi(\cdot|x)}[Q_\psi(x,u;\tau)]$.
At the episode level we use the Rockafellar--Uryasev CVaR surrogate~\cite{RockafellarUryasev2000} (or its quantile approximation) and train with Monte Carlo estimates.

\paragraph{CVaR-weighted advantage and PPO updates.}
Let $V_\alpha^\pi(x)\!\approx\!\mathrm{CVaR}_\alpha(L\,|\,x)$ and define a \emph{CVaR-weighted} generalized advantage
\begin{equation}
\label{eq:cvar-gae}
A_t^{(\alpha)} \;\approx\; \sum_{l=0}^{\infty} (\gamma\lambda)^l \Big(\ell_{t+l}-\widehat{V}_\alpha(x_{t+l}) + \gamma\,\widehat{V}_\alpha(x_{t+l+1})\Big),
\end{equation}
where $\ell_t$ is a one-step loss (including spread, temporary impact, and transient slippage) and the structure mirrors GAE~\cite{Schulman2016GAE} with a CVaR baseline.
The actor uses PPO~\cite{Schulman2017PPO} with KL and entropy regularization:
\begin{equation}
\label{eq:actor-loss}
\mathcal{L}_{\text{actor}}
=\mathbb{E}_t\!\left[ \min\!\Big(r_t(\theta)\,A_t^{(\alpha)},\;\mathrm{clip}(r_t(\theta),1-\epsilon,1+\epsilon)\,A_t^{(\alpha)}\Big)\right]
\;+\;\lambda_{\mathrm{KL}}\,\mathbb{E}_x\big[\mathrm{KL}(\pi_\theta(\cdot|x)\,\|\,\pi_{\mathrm{ref}}(\cdot|x))\big]
\;+\;\lambda_{\mathrm{ent}}\,\mathbb{E}_x[\mathcal{H}(\pi_\theta(\cdot|x))],
\end{equation}
with $r_t(\theta)=\pi_\theta(u_t|x_t)/\pi_{\theta_{\mathrm{old}}}(u_t|x_t)$ and $\pi_{\mathrm{ref}}$ an EMA reference policy. 
The KL term serves as a trust-region proxy~\cite{Schulman2015TRPO,Achiam2017CPO} and admits a DRO interpretation (see \S4). 
The critic minimizes the quantile Huber loss for $Q_\psi$~\cite{Dabney2018IQN}.

\paragraph{Implementation notes.}
Trajectories are collected \emph{through} the safety filter (Sec.~\ref{sec:method:cbf}), i.e., the actor proposes $u_t^{\mathrm{nom}}$ which is minimally corrected by the CBF--QP into $u_t$. 
We log solver telemetry (active constraints, tightness, rate utilization, gate scores, slack, status/time) to support both training diagnostics and audit trails.
KL and entropy coefficients can be scheduled to avoid premature collapse and to match the tightening of $\alpha$~\cite{Schulman2017PPO}.

\subsection{Tail-Coverage Controller: Temperature Sampling and Tail-Boost}
\label{sec:method:coverage}

\paragraph{Motivation.}
For small $\alpha$ (e.g., $1\%\!-\!5\%$), uniform quantile sampling yields few tail samples and high variance, destabilizing training.
We therefore introduce \emph{temperature-tilted} quantile sampling and an explicit \emph{tail-boost}, combined with a PID controller to track a target \emph{effective tail mass}.

\paragraph{Temperature-tilted sampling and importance weights.}
Define the sampling density over $\tau\in[0,1]$:
\begin{equation}
\label{eq:tail-sampler}
p_T(\tau) \;\propto\; e^{-\tau/T},\qquad T\in[T_{\min},T_{\max}],
\end{equation}
so that smaller $T$ emphasizes low quantiles. 
With $\tau_k\!\sim\!p_T$, we use self-normalized importance weights $w_k \propto 1/p_T(\tau_k)$ to form an unbiased CVaR estimator (concentration bounds in \S4; see also standard results on self-normalized importance sampling~\cite{Owen2013MC}). 
Additionally, for $\tau\le\alpha$ we assign a \emph{tail-boost} factor $\gamma_{\mathrm{tail}}\!\ge\!1$ to increase the effective tail count.

\paragraph{Coverage metric and PID tracking.}
Let the \emph{effective tail mass} within a minibatch be
\[
\widehat{w}\;=\;\frac{1}{K}\sum_{k=1}^{K}\mathbf{1}\{\tau_k\le\alpha\}.
\]
Given a target $w_{\mathrm{target}}$ (e.g., $1.5\alpha$), define error $e=w_{\mathrm{target}}-\widehat{w}$ and update $(T,\gamma_{\mathrm{tail}})$ via discrete PID (clipped to feasible ranges):
\begin{align}
\label{eq:pid-update}
T_{n+1}&=\mathrm{clip}\!\left(T_n + \kappa_P e_n + \kappa_I \sum_{j=1}^{n} e_j + \kappa_D(e_n-e_{n-1}),\;T_{\min},T_{\max}\right),\\
\gamma_{\mathrm{tail},\,n+1}&=\mathrm{clip}\!\left(\gamma_{\mathrm{tail},\,n} + \eta_P e_n + \eta_I \sum_{j=1}^{n} e_j + \eta_D(e_n-e_{n-1}),\;\gamma_{\min},\gamma_{\max}\right).
\end{align}
PID design and anti-windup follow classical practice~\cite{AstromHagglund1995PID}. 
We employ an $\alpha$ schedule that tightens from a permissive level (e.g., $0.10$) toward the target (e.g., $0.025$), while the controller stabilizes $\widehat{w}$; the PPO KL penalty can be increased in tandem to limit policy drift~\cite{Schulman2017PPO,Schulman2015TRPO,Achiam2017CPO}.

\subsection{White-Box CBF--QP Safety Layer}
\label{sec:method:cbf}

\paragraph{Discrete-time CBF constraints.}
Let $h_i(x)\!\ge\!0$ denote the $i$-th safety function and consider a local affine state update $x_{t+1}=f(x_t)+g(x_t)u_t$.
We enforce discrete-time CBF conditions~\cite{Ames2017CBF_TAC,Ames2019CBF_Tutorial}:
\begin{equation}
\label{eq:cbf-dt}
h_i\!\big(f(x_t)+g(x_t)u_t\big) \;-\; (1-\kappa_i\Delta t)\,h_i(x_t) \;\ge\; -\,\zeta_i,
\qquad \zeta_i\ge 0,\;\kappa_i>0,
\end{equation}
where slack variables are heavily penalized and only activated when unavoidable (robust margins are analyzed in \S4).

\paragraph{Ellipsoidal NTB, box/rate limits, and a sign-consistency gate.}
Let $b^\star(x)$ be a target exposure vector (e.g., delta/vega) and $b(x,u)$ the exposure after action $u$.
Define the \emph{ellipsoidal no-trade band} (NTB)
\begin{equation}
\label{eq:ntb}
e(x,u) \;=\; b(x,u)-b^\star(x),
\qquad e^\top M e \;\le\; b_{\max},\quad M\succ 0,
\end{equation}
and \emph{box/rate} limits $u_{\min}\!\le\!u_t\!\le\!u_{\max}$, $\|u_t-u_{t-1}\|_2\!\le\!r_{\max}$.
We further require a \emph{sign-consistency gate}
\begin{equation}
g_{\mathrm{cons}}(x,u)\;=\;\min_{j=1,\dots,J}\,\langle u,\widehat{\nabla}\Pi^{(j)}(x)\rangle \;-\; \delta_{\mathrm{adv}} \;\ge\; 0,
\end{equation}
so that trades align with an ensemble of interpretable signals (e.g., advantage-proxy gradients from the distributional critic or pricing/hedging sensitivities)~\cite{SuttonBarto2018}. 
Near expiry or under extreme volatility, we shrink $b_{\max}\!\leftarrow\!\eta_b b_{\max}$ and tighten $r_{\max}\!\leftarrow\!\eta_r r_{\max}$ with $\eta_b,\eta_r\in(0,1)$ to improve feasibility and stability.

\paragraph{QP formulation and minimal-deviation projection.}
Given the actor’s proposal $u_t^{\mathrm{nom}}$, we compute the closest safe action $u_t$ by solving the convex QP
\begin{align}
\label{eq:qp-safety}
\min_{u_t,\;\zeta\ge 0}\quad 
& \frac{1}{2}\,(u_t-u_t^{\mathrm{nom}})^\top H\,(u_t-u_t^{\mathrm{nom}}) \;+\; c^\top u_t \;+\; \rho\,\|\zeta\|_1 \\
\text{s.t.}\quad 
& \text{CBF: } h_i(f(x_t)+g(x_t)u_t) - (1-\kappa_i\Delta t)\,h_i(x_t) \ge -\zeta_i,\;\forall i, \nonumber\\
& \text{NTB: } e(x_t,u_t)^\top M e(x_t,u_t) \le b_{\max}, \qquad 
  \text{Box/Rate: } u_{\min}\!\le\!u_t\!\le\!u_{\max},\;\|u_t-u_{t-1}\|_2\le r_{\max}, \nonumber\\
& \text{Gate: } g_{\mathrm{cons}}(x_t,u_t)\ge 0.\nonumber
\end{align}
Here $H\!\succ\!0$ defines the deviation metric, $c$ encodes linear trading frictions, and $\rho\!\gg\!0$ penalizes any slack.
We use OSQP~\cite{Stellato2020OSQP} with warm starts for efficiency and robustness; see \cite{BoydVandenberghe2004} for background on convex QPs. 
When $\zeta=0$,~\eqref{eq:cbf-dt} implies forward invariance of the safe set; the quadratic objective makes $u_t$ the $H$-metric projection of $u_t^{\mathrm{nom}}$ onto the feasible set (formalized in \S4).

\paragraph{Telemetry for auditability and operations.}
For each step, the solver returns: \texttt{active\_set} (indices of active constraints), \texttt{tightest\_id}, \texttt{rate\_util} $=\|u_t{-}u_{t-1}\|_2/r_{\max}$, \texttt{gate\_score} $=g_{\mathrm{cons}}(x_t,u_t)$, \texttt{slack\_sum} $=\|\zeta\|_1$, and \texttt{solver\_status/time}. 
We penalize nonzero slack or non-optimal statuses in the RL reward and log incidents for post-hoc audit, closing the loop between \emph{explainable interception} and \emph{governance}.

\paragraph{Pseudocode: training loop and safety filter.}

\begin{algorithm}[H]
\caption{Tail-Safe IQN--CVaR--PPO (on-policy training)}
\label{alg:tailsafe_train}
\begin{algorithmic}[1]
\STATE Initialize actor $\theta$, critic $\psi$, reference policy $\pi_{\mathrm{ref}}\!\leftarrow\!\pi_\theta$, temperature $T$, tail-boost $\gamma_{\mathrm{tail}}$, and target coverage $w_{\mathrm{target}}$.
\FOR{iterations $k=1,2,\dots$}
  \STATE Collect trajectories using the safety filter (Alg.~\ref{alg:cbf_filter}) to obtain $\{(x_t,u_t,r_t,\text{telemetry}_t)\}$.
  \STATE Sample quantiles $\tau_k\!\sim\!p_T$ (Eq.~\eqref{eq:tail-sampler}) and apply tail-boost to $\tau\!\le\!\alpha$.
  \STATE Update critic $Q_\psi$ by minimizing the quantile Huber loss (IQN).
  \STATE Estimate $V_\alpha$ and $A^{(\alpha)}$ via Eqs.~\eqref{eq:cvar-quantile}--\eqref{eq:cvar-gae}; update actor by minimizing Eq.~\eqref{eq:actor-loss}.
  \STATE Compute $\widehat{w}=\frac{1}{K}\sum \mathbf{1}\{\tau\le\alpha\}$ and update $(T,\gamma_{\mathrm{tail}})$ using the PID rules~\eqref{eq:pid-update} (with clipping).
  \STATE Tighten $\alpha$ according to a schedule; update $\pi_{\mathrm{ref}}$ via EMA; log policy KL and telemetry summaries.
\ENDFOR
\end{algorithmic}
\end{algorithm}

\begin{algorithm}[H]
\caption{CBF--QP safety filter (per step)}
\label{alg:cbf_filter}
\begin{algorithmic}[1]
\STATE Inputs: $x_t$, proposed action $u_t^{\mathrm{nom}}$, previous action $u_{t-1}$, params $(H,M,b_{\max},r_{\max},u_{\min},u_{\max},\kappa,\Delta t)$.
\STATE Formulate QP~\eqref{eq:qp-safety} with discrete CBF, NTB, box/rate limits, and sign-consistency gate; near expiry, shrink $b_{\max}\!\leftarrow\!\eta_b b_{\max}$ and $r_{\max}\!\leftarrow\!\eta_r r_{\max}$.
\STATE Solve the QP (warm-start) to obtain $u_t$, active set, tightest constraint, slack $\zeta$, and solver status/time.
\STATE Emit telemetry: \texttt{active\_set}, \texttt{tightest\_id}, \texttt{rate\_util}, \texttt{gate\_score}, \texttt{slack\_sum}, \texttt{solver\_status/time}.
\STATE If $\zeta>0$ or status $\neq$ optimal, add a penalty to the RL reward and log the event; otherwise execute $u_t$ and advance the environment to $x_{t+1}$.
\STATE Return $u_t$ and telemetry.
\end{algorithmic}
\end{algorithm}

\section{Theoretical Results}
\label{sec:theory}

We formalize guarantees for the proposed \textbf{Tail-Safe} framework. 
Our results cover (i) \emph{robust forward invariance} of the discrete-time CBF constraints under bounded model mismatch, 
(ii) the \emph{minimal-deviation} nature of the QP safety layer, 
(iii) a \emph{KL--DRO} upper bound linking per-state KL regularization to distributional robustness, 
(iv) \emph{concentration} and sample-complexity of the temperature-tilted CVaR estimator with a coverage controller, 
(v) a \emph{trust-region improvement} inequality for CVaR with KL-limited policy updates, 
(vi) \emph{feasibility persistence} under tail guards, and 
(vii) \emph{negative-advantage suppression} induced by the sign-consistency gate.
Complete proofs are deferred to \textbf{Appendix~A} (with subsections indicated after each result).

Throughout, let $\|\cdot\|$ denote the Euclidean norm, $\langle\cdot,\cdot\rangle$ the Euclidean inner product, and $\mathbb{B}_r(x)$ the closed ball of radius $r$. 
We reuse notation from Sections~\ref{sec:prelim}--\ref{sec:method}.

\subsection*{Theorem 1 (robust forward invariance of the safety set)}
\begin{assumption}[Local dynamics and mismatch]
\label{assump:dynamics}
There exist locally Lipschitz functions $f,g$ and an additive disturbance $w_t$ such that the state update obeys
$x_{t+1} = f(x_t) + g(x_t) u_t + w_t$ with $\|w_t\|\le \bar w$ almost surely. 
For each barrier $h_i:\mathbb{R}^d\to\mathbb{R}$ there is an $L_i>0$ with $|h_i(x)-h_i(y)|\le L_i\|x-y\|$.
\end{assumption}

\begin{assumption}[Discrete-time CBF constraint with margin]
\label{assump:cbf}
At time $t$, the QP safety layer enforces for each $i$:
\[
h_i\!\big(f(x_t)+g(x_t)u_t\big) - (1-\kappa_i\Delta t)\,h_i(x_t) \;\ge\; \varepsilon_i,
\qquad \kappa_i>0,
\]
with margin $\varepsilon_i \ge L_i \bar w$.
\end{assumption}

\begin{theorem}[Robust forward invariance]
\label{thm:invariance}
Under Assumptions~\ref{assump:dynamics}--\ref{assump:cbf}, if $h_i(x_t)\ge 0$ for all $i$, then $h_i(x_{t+1})\ge 0$ for all $i$. 
Hence the safe set $\mathcal{C}:=\{x:\,h_i(x)\ge 0,\,\forall i\}$ is forward invariant. 
\end{theorem}

\paragraph{Proof sketch.} Using Lipschitzness,
$h_i(x_{t+1}) \ge h_i(f+gu_t)-L_i\|w_t\|\ge (1-\kappa_i\Delta t)h_i(x_t) + \varepsilon_i - L_i\bar w \ge 0$.
A full inductive argument is given in \textbf{Appendix~A.1}. 
(See also robust CBF analyses such as \cite{Jankovic2018RobustCBF,Xu2015RobustCBF}.)

\subsection*{Proposition 1 (minimal-deviation $H$-metric projection)}
\begin{assumption}[Convex feasibility]
\label{assump:convex}
For fixed $x_t$, the feasible action set $\mathcal{S}(x_t)$ induced by the constraints in~\eqref{eq:cbf-dt}, \eqref{eq:ntb}, and the box/rate and gate constraints is nonempty, closed, and convex in $u$ (this holds under affine-in-$u$ CBF surrogates/linearizations and convex NTB/box/rate/gate specifications).
\end{assumption}

\begin{proposition}[Shifted projection]
\label{prop:projection}
Let $H\!\succ\!0$, $c\in\mathbb{R}^m$, and $\rho$ be sufficiently large so that the QP~\eqref{eq:qp-safety} is solved with $\zeta=0$. 
Then its unique optimizer $u_t^\star$ satisfies
\[
u_t^\star \;=\; \arg\min_{u\in\mathcal{S}(x_t)} \frac{1}{2}\,\|u - (u_t^{\mathrm{nom}} - H^{-1}c)\|_{H}^2,
\]
i.e., $u_t^\star$ is the $H$-metric projection of the \emph{shifted anchor} $u_t^{\mathrm{nom}} - H^{-1}c$ onto $\mathcal{S}(x_t)$.
\end{proposition}

\paragraph{Proof sketch.} Completing the square yields 
$\tfrac{1}{2}(u-u^{\mathrm{nom}})^\top H (u-u^{\mathrm{nom}}) + c^\top u
= \tfrac{1}{2}\|u-(u^{\mathrm{nom}}-H^{-1}c)\|_H^2 + \mathrm{const}$.
Strict convexity plus convex feasibility implies uniqueness; KKT conditions characterize the projection. 
Details are in \textbf{Appendix~A.2} (cf.\ \cite{BoydVandenberghe2004}).

\subsection*{Theorem 2 (KL--DRO upper bound and conservatism of per-state KL)}
\begin{assumption}[CVaR surrogate]
\label{assump:cvar}
For $\alpha\in(0,1)$ and any threshold $t\in\mathbb{R}$, define $\phi_t(z)=(z-t)_+$ and 
$\mathrm{CVaR}_\alpha(L)=\min_{t}\,t+\tfrac{1}{\alpha}\,\mathbb{E}[\phi_t(L)]$~\cite{RockafellarUryasev2000}.
\end{assumption}

\begin{assumption}[Path-wise KL radius and occupancy control]
\label{assump:kl}
Let $\mathcal{Q}_\rho=\{Q:\mathrm{KL}(Q\|P)\le \rho\}$ be a KL ball around a reference path distribution $P$ (the simulator/behavior distribution). 
Assume per-state policy KL is bounded: $\mathrm{KL}(\pi'(\cdot|x)\|\pi_{\mathrm{ref}}(\cdot|x))\le \beta$ for all $x$ in the support, 
and the induced pathwise KL satisfies $\rho \le C_\mathrm{occ}\,\beta$ for some constant $C_\mathrm{occ}$ depending on the horizon and mixing/occupancy properties (cf.\ Pinsker-type arguments and standard occupancy coupling).
\end{assumption}

\begin{theorem}[Donsker--Varadhan bound and per-state KL conservatism]
\label{thm:kldro}
For any $\eta>0$,
\[
\sup_{Q\in\mathcal{Q}_\rho}\mathrm{CVaR}_\alpha(L) 
\;\le\; \min_{t\in\mathbb{R}} \left\{\, t + \frac{1}{\alpha\eta}\Big(\rho + \log \mathbb{E}_{P}\!\big[e^{\eta\,\phi_t(L)}\big]\Big)\right\}.
\]
Moreover, under Assumption~\ref{assump:kl} the RHS is upper-bounded by the same expression with $\rho$ replaced by $C_\mathrm{occ}\beta$. 
Hence, penalizing per-state KL by $\lambda_{\mathrm{KL}}\!\cdot\!\mathbb{E}_x[\mathrm{KL}(\pi(\cdot|x)\|\pi_{\mathrm{ref}}(\cdot|x))]$ controls a KL--DRO upper bound on the CVaR surrogate and thus quantifies conservatism.
\end{theorem}

\paragraph{Proof sketch.} Apply the Donsker--Varadhan variational inequality to $\mathbb{E}_Q[\phi_t(L)]$~\cite{DonskerVaradhan1975}, then minimize over $t$. 
Relate pathwise KL to the expected per-state KL via occupancy coupling and Pinsker’s inequality (see, e.g., \cite{Schulman2015TRPO,Achiam2017CPO,CsiszarKorner2011,NamkoongDuchi2017}). 
Appendix \textbf{A.3} provides details.

\subsection*{Theorem 3 (bias/variance and sample complexity of the CVaR estimator)}
\begin{assumption}[Temperature sampling, bounded importance weights, bounded loss]
\label{assump:is}
Quantiles are sampled from $p_T(\tau)\propto e^{-\tau/T}$ on $[0,1]$ with $T\in[T_{\min},T_{\max}]$ and $p_T(\tau)\ge p_{\min}>0$. 
Self-normalized importance weights are $w_k\propto 1/p_T(\tau_k)$ (normalized within the batch). 
Losses are almost surely bounded: $|L|\le B$.
\end{assumption}

Define the self-normalized estimator (with tail-boost implemented by oversampling, absorbed into $p_T$):
\[
\widehat{\mathrm{CVaR}}_\alpha
=\frac{\sum_{k=1}^K w_k\,\mathbf{1}\{\tau_k\le\alpha\}\,L^{(\tau_k)}}
{\sum_{k=1}^K w_k\,\mathbf{1}\{\tau_k\le\alpha\}},
\quad 
\alpha_{\mathrm{eff}}:=\mathbb{E}[\mathbf{1}\{\tau\le\alpha\}]\,.
\]

\begin{theorem}[Concentration of the temperature-tilted CVaR estimator]
\label{thm:cvar-concentration}
Under Assumption~\ref{assump:is}, for any $\delta\in(0,1)$, with probability at least $1-\delta$,
\[
\big|\widehat{\mathrm{CVaR}}_\alpha - \mathrm{CVaR}_\alpha\big|
\;\le\; 
\underbrace{C_1\,B\,\sqrt{\frac{\log(2/\delta)}{K\,\alpha_{\mathrm{eff}}}}}_{\text{variance term}}
\;+\;
\underbrace{C_2\,B\,\big|\alpha_{\mathrm{eff}}-\alpha\big|}_{\text{coverage mismatch}},
\]
for absolute constants $C_1,C_2$ depending on $p_{\min}$ and the self-normalization scheme. 
In particular, the PID controller that tracks $\alpha_{\mathrm{eff}}\!\approx\!w_{\mathrm{target}}$ (with $w_{\mathrm{target}}$ close to $\alpha$) reduces the coverage-mismatch bias and improves the rate constant in the variance term.
\end{theorem}

\paragraph{Proof sketch.} Combine self-normalized importance sampling concentration (e.g., empirical Bernstein/Hoeffding-style bounds for ratio estimators) with bounded weights and losses~\cite{Owen2013MC}. 
A first-order expansion quantifies the effect of replacing $\alpha$ by $\alpha_{\mathrm{eff}}$; the controller reduces this mismatch. 
See \textbf{Appendix~A.4} for a full derivation.

\subsection*{Theorem 4 (trust-region improvement inequality for CVaR with KL limits)}
\begin{assumption}[Per-state KL constraint and smoothness]
\label{assump:tr}
For a policy update $\pi'\!=\!\pi+\Delta$, assume $\mathrm{KL}(\pi'(\cdot|x)\|\pi(\cdot|x))\le \beta$ for all $x$ and that the one-step loss $\ell$ is $L$-Lipschitz in actions and states under the dynamics in Assumption~\ref{assump:dynamics}. 
Let $J_\alpha(\pi)=\mathrm{CVaR}_\alpha(L_T(\pi))$ for horizon $T$.
\end{assumption}

\begin{theorem}[CVaR trust-region improvement]
\label{thm:trpo-cvar}
There exists a constant $C_\alpha>0$ (depending on $L$, $T$, and loss bounds) such that
\[
J_\alpha(\pi') \;\le\; J_\alpha(\pi) 
\;+\; \mathbb{E}_{x\sim d_\pi,\,u\sim\pi}\!\big[\omega(x,u)\,\tilde A_\pi^{(\alpha)}(x,u)\big]
\;+\; C_\alpha\,\sqrt{\beta} \;+\; o(\|\Delta\|),
\]
where $\omega=\pi'(\cdot|x)/\pi(\cdot|x)$ and $\tilde A_\pi^{(\alpha)}$ is the CVaR-weighted advantage (cf.\ \eqref{eq:cvar-gae}). 
Thus, a small KL radius $\beta$ controls the degradation term, yielding a trust-region style guarantee for the CVaR objective.
\end{theorem}

\paragraph{Proof sketch.} Adapt the performance-difference lemma to the CVaR surrogate (replace value with CVaR baseline), then bound the state-distribution shift by total variation and Pinsker’s inequality $\mathrm{TV}\le\sqrt{\tfrac{1}{2}\mathrm{KL}}$~\cite{Schulman2015TRPO,Achiam2017CPO,CsiszarKorner2011}. 
Smoothness of $\ell$ and the horizon accumulation yield $C_\alpha\sqrt{\beta}$. 
A detailed derivation is given in \textbf{Appendix~A.5}.

\subsection*{Theorem 5 (feasibility persistence under tail guards)}
\begin{assumption}[Lipschitz constraints and affine exposure map]
\label{assump:feas}
Assume $e(x,u)=A(x)u-d(x)$ with $A$, $d$ locally Lipschitz, $M\succ 0$, and that box/rate sets are convex. 
Let the CBF surrogates used in the QP be affine in $u$ for fixed $x$.
\end{assumption}

\begin{assumption}[Margins at time $t$]
\label{assump:margins}
Suppose at time $t$ the NTB and rate constraints hold with margins 
$e(x_t,u_t)^\top M e(x_t,u_t) \le b_{\max}-\delta_b$ and $\|u_t-u_{t-1}\|_2 \le r_{\max}-\delta_r$ for some $\delta_b,\delta_r>0$.
\end{assumption}

\begin{theorem}[Persistence via NTB shrinkage and rate tightening]
\label{thm:feas-persist}
Under Assumptions~\ref{assump:dynamics}, \ref{assump:feas}, and \ref{assump:margins}, 
there exist shrinkage factors $\eta_b,\eta_r\in(0,1)$, computable from local Lipschitz constants of $(A,d,f,g)$ and $(\kappa_i)$ and the disturbance bound $\bar w$, such that replacing $b_{\max}\!\leftarrow\!\eta_b b_{\max}$ and $r_{\max}\!\leftarrow\!\eta_r r_{\max}$ guarantees that the QP at time $t{+}1$ remains feasible with $\zeta=0$. 
Equivalently, the feasible set intersection (CBF, NTB, box, rate, gate) remains nonempty at $t{+}1$.
\end{theorem}

\paragraph{Proof sketch.} Use tube arguments: with Lipschitz dynamics and constraints, the next-state feasible set contains a ball around the previous safe action projected by the rate set. 
Choosing $(\eta_b,\eta_r)$ to upper bound drift induced by $(x_t\!\to\!x_{t+1})$ and disturbance keeps the intersection nonempty. 
Full details appear in \textbf{Appendix~A.6} (see also viability arguments in robust CBF literature~\cite{Jankovic2018RobustCBF}).

\subsection*{Proposition 2 (negative-advantage suppression by sign-consistency)}
\begin{assumption}[Gate alignment and mismatch]
\label{assump:gate}
For each $x$, let $g_{\mathrm{cons}}(x,u)=\min_{j\le J}\langle u,\widehat{\nabla}\Pi^{(j)}(x)\rangle - \delta_{\mathrm{adv}}$ with $\delta_{\mathrm{adv}}\ge 0$. 
Assume there exists a unit vector $v(x)$ such that 
$\|\widehat{\nabla}\Pi^{(j)}(x) - v(x)\| \le \epsilon_g$ for all $j$, and that the CVaR-weighted advantage satisfies a local linearization 
$\tilde A_\pi^{(\alpha)}(x,u) \approx \langle u, \nabla \tilde A_\pi^{(\alpha)}(x)\rangle$ with $\angle(\nabla \tilde A_\pi^{(\alpha)}(x),v(x))\le \epsilon_\theta$.
\end{assumption}

\begin{proposition}[Gate-induced lower bound]
\label{prop:gate}
Under Assumption~\ref{assump:gate}, any action $u$ passing the gate ($g_{\mathrm{cons}}(x,u)\ge 0$) obeys
\[
\mathbb{E}\!\left[\tilde A_\pi^{(\alpha)}(x,u)\,\middle|\,x\right] \;\ge\; -\,\xi(\epsilon_g,\epsilon_\theta,\delta_{\mathrm{adv}},\|u\|),
\]
for an explicit function $\xi$ that vanishes as $(\epsilon_g,\epsilon_\theta)\!\to\!0$ and increases with $\delta_{\mathrm{adv}}$ and $\|u\|$. 
In particular, for sufficiently small alignment errors the gate suppresses negative CVaR-advantage trades up to a controlled tolerance.
\end{proposition}

\paragraph{Proof sketch.} 
Gate feasibility implies $\langle u, v(x)\rangle \ge \delta_{\mathrm{adv}} - \|u\|\epsilon_g$. 
Using the angle bound between $v(x)$ and $\nabla \tilde A_\pi^{(\alpha)}(x)$ and Cauchy--Schwarz yields 
$\langle u, \nabla \tilde A_\pi^{(\alpha)}(x)\rangle \ge \|u\|(\delta_{\mathrm{adv}}/\|u\| - \epsilon_g)\cos\epsilon_\theta - \|u\|\,\mathcal{O}(\epsilon_g)$, 
which provides the claimed lower bound. 
Formal constants are derived in \textbf{Appendix~A.7}.

\medskip
\noindent\textbf{Remarks.}
(i) Theorems~\ref{thm:invariance} and \ref{thm:feas-persist} formalize \emph{hard safety} by construction plus feasibility resilience under tail guards, connecting solver telemetry to auditable interventions. 
(ii) Theorems~\ref{thm:kldro} and \ref{thm:trpo-cvar} justify the PPO+KL design for tail-risk optimization: per-state KL acts as a DRO-style conservatism control while ensuring trust-region stability. 
(iii) Theorem~\ref{thm:cvar-concentration} shows why the coverage controller reduces variance and bias in CVaR estimation as $\alpha$ tightens.

\section{Experiments in Arbitrage-Free Synthetic Markets}
\label{sec:experiments}

We evaluate \textbf{Tail-Safe} in a synthetic yet finance-grounded environment that is (i) \emph{arbitrage-free} by construction (SSVI$\!\rightarrow$Dupire$\!\rightarrow$VIX), (ii) \emph{microstructure-aware} through ABIDES/MockLOB execution with spread, temporary, and transient impact, and (iii) \emph{stressable} along interpretable dimensions (level/slope/curvature of the volatility surface, equity--volatility correlation, impact strength/decay, and time-to-expiry).
All results are reported with \emph{unified sample sizes} and \emph{multiple random seeds}, with uncertainty quantified by paired bootstrap confidence intervals~\cite{Efron1979Bootstrap,DavisonHinkley1997}.
This section details the protocol, metrics, results, ablations, and safety telemetry.

\subsection{Protocol}
\label{sec:exp:protocol}

\paragraph{ID/OOD split and stress dimensions.}
We generate a calibrated SSVI surface family $w(k,\tau)$ satisfying static no-arbitrage (Sec.~\ref{sec:ssvi}), transform it to Dupire local volatility (Sec.~\ref{sec:dupire}), and construct a surface-consistent 30D VIX leg (Sec.~\ref{sec:vix}).
We define \emph{in-distribution} (ID) scenarios by sampling around the baseline calibration and \emph{out-of-distribution} (OOD) scenarios by perturbing:
\begin{enumerate}
    \item \textbf{Level/slope/curvature} of the SSVI parameters $(\theta,\varphi,\rho)$ across maturities;
    \item \textbf{Equity--volatility correlation} regimes (including extreme negative spikes);
    \item \textbf{Impact strength/decay} in the execution model (spread, temporary coefficient $\eta$, transient kernel $G$);
    \item \textbf{Time-to-expiry}, emphasizing near-expiry regimes where greeks and liquidity risks intensify.
\end{enumerate}
All stressed configurations remain free of \emph{static} arbitrage by construction.

\paragraph{Unified sample sizes and multiple seeds.}
To ensure fair comparisons, each method/variant is evaluated on the \emph{same} set of scenario seeds and the same number of paths per seed. 
When legacy artifacts contain differing sample sizes (e.g., $n{=}400$ vs.\ $n{=}200$), we \emph{re-subsample/pair} to a common effective $n$ per seed, and we report all statistics with \emph{paired} uncertainty across methods.
This protocol reduces variance inflation and follows best practices in RL evaluation~\cite{Henderson2018Matters}.

\paragraph{Training and evaluation schedules.}
Policies are trained with the IQN--CVaR--PPO learner (Sec.~\ref{sec:method:rl}) using the Tail-Coverage Controller (Sec.~\ref{sec:method:coverage}) and the white-box CBF--QP safety layer (Sec.~\ref{sec:method:cbf}) \emph{active during data collection and evaluation}.
The CVaR level $\alpha$ tightens from a permissive value (e.g., $0.10$) toward the target (e.g., $0.025$) according to a cosine schedule; KL and entropy coefficients are co-scheduled to limit policy drift.
All wall-clock budgets, iteration counts, and solver tolerances are reported in the reproducibility checklist (Appendix~C).

\subsection{Metrics}
\label{sec:exp:metrics}

\paragraph{Tail risk and central performance.}
We report absolute-loss $\operatorname{VaR}_{\alpha}$ and $\operatorname{ES}_{\alpha}$ (Sec.~\ref{sec:risk}), along with mean, standard deviation, and portfolio-style ratios:
\begin{align}
\text{Sharpe} &= \frac{\mathbb{E}[\Pi]}{\sqrt{\mathrm{Var}(\Pi)}}, 
\qquad\text{(zero risk-free benchmark)}~\cite{Sharpe1994},\\
\text{Sortino} &= \frac{\mathbb{E}[\Pi]}{\sqrt{\mathbb{E}[\min(\Pi,0)^2]}}, \qquad\text{(downside deviation)}~\cite{SortinoPrice1994},\\
\Omega(\tau_0) &= \frac{\int_{\tau_0}^{\infty} \big(1-F_\Pi(z)\big)\,dz}{\int_{-\infty}^{\tau_0} F_\Pi(z)\,dz}, \quad \tau_0{=}0, \qquad\text{(Omega ratio)}~\cite{KeatingShadwick2002}.
\end{align}
We visualize distributions with ECDFs and histograms and annotate tail quantiles.

\paragraph{Stability and regularization.}
We track (i) \emph{policy KL step} $\mathbb{E}_x[\mathrm{KL}(\pi(\cdot|x)\|\pi_{\mathrm{ref}}(\cdot|x))]$ per update, (ii) \emph{effective tail coverage} $\widehat{w}$ relative to the target $w_{\mathrm{target}}$, and (iii) \emph{entropy} of the policy. 

\paragraph{Safety telemetry.}
From the CBF--QP solver, we log \texttt{active\_set}, \texttt{tightest\_id}, \texttt{rate\_util}, \texttt{gate\_score}, \texttt{slack\_sum}, and \texttt{solver\_status/time} (Sec.~\ref{sec:method:cbf}), and report:
(i) frequency of the tightest constraints, (ii) slack mass and feasibility rate, and (iii) median/P95 solver time per step.

\paragraph{Uncertainty quantification and tests.}
For each statistic we report $95\%$ paired bootstrap CIs across scenario paths and seeds~\cite{Efron1979Bootstrap,DavisonHinkley1997}. 
When multiple metrics are tested simultaneously, $p$-values are FDR-adjusted using Benjamini--Hochberg~\cite{BenjaminiHochberg1995}. 
Effect sizes for pairwise comparisons are summarized via the Vargha--Delaney $\hat{A}_{12}$ statistic~\cite{VarghaDelaney2000}.

\subsection{Results}
\label{sec:exp:results}

\paragraph{Distributional evidence.}
Figure~\ref{fig:overlay-ecdf} overlays ECDFs of P\&L; \textbf{Tail-Safe} shifts the $1$--$3\%$ left tail rightwards while keeping the central mass comparable to the QP-only baseline (consistent with the CVaR objective).
Figure~\ref{fig:overlay-hist} shows thinner negative tails for \textbf{Tail-Safe}; premium-normalized variants (not shown) display narrower dispersion.

\begin{figure*}[t]
  \centering
  \includegraphics[width=0.8\textwidth]{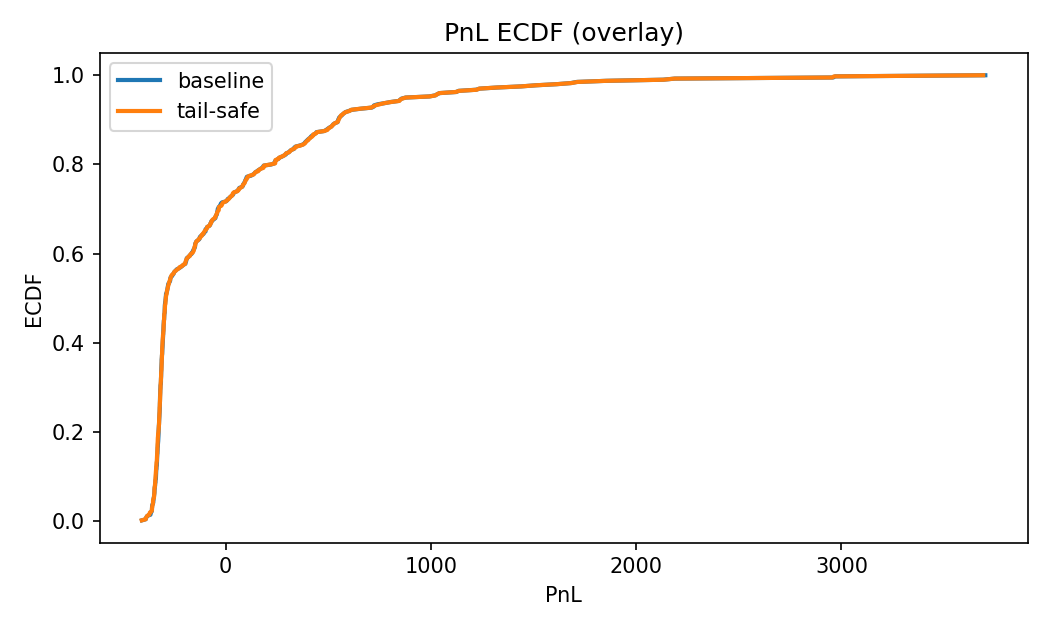}
  \caption{\textbf{PnL ECDF (overlay).} \textbf{Tail-Safe} vs.\ QP-only baseline on matched scenarios and unified sample sizes. Shaded bands (if present) indicate $95\%$ paired bootstrap CIs at each quantile (Appendix~B).}
  \label{fig:overlay-ecdf}
\end{figure*}

\begin{figure*}[t]
  \centering
  \includegraphics[width=0.8\textwidth]{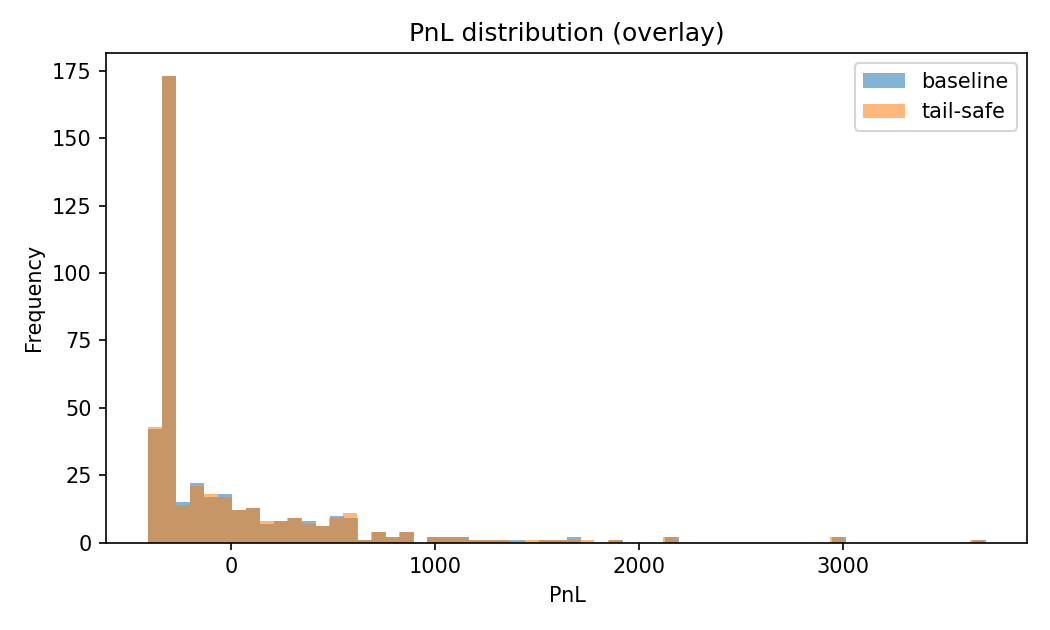}
  \caption{\textbf{PnL distribution (overlay).} Negative tails thin out under \textbf{Tail-Safe}, while the bulk remains comparable. Vertical lines annotate $\operatorname{VaR}_\alpha$ and $\operatorname{ES}_\alpha$.}
  \label{fig:overlay-hist}
\end{figure*}

\paragraph{Per-variant panels.}
Figure~\ref{fig:baseline-panels} reports ECDF/histogram for the QP-only baseline; Figure~\ref{fig:tailsafe-panels} reports the same for \textbf{Tail-Safe}. 
In both, tail markers highlight $\operatorname{VaR}_{\alpha}$ and $\operatorname{ES}_{\alpha}$ at the evaluation $\alpha$ level.

\begin{figure}[t]
  \centering
  \begin{subfigure}[t]{0.48\linewidth}
    \centering
    \includegraphics[width=\linewidth]{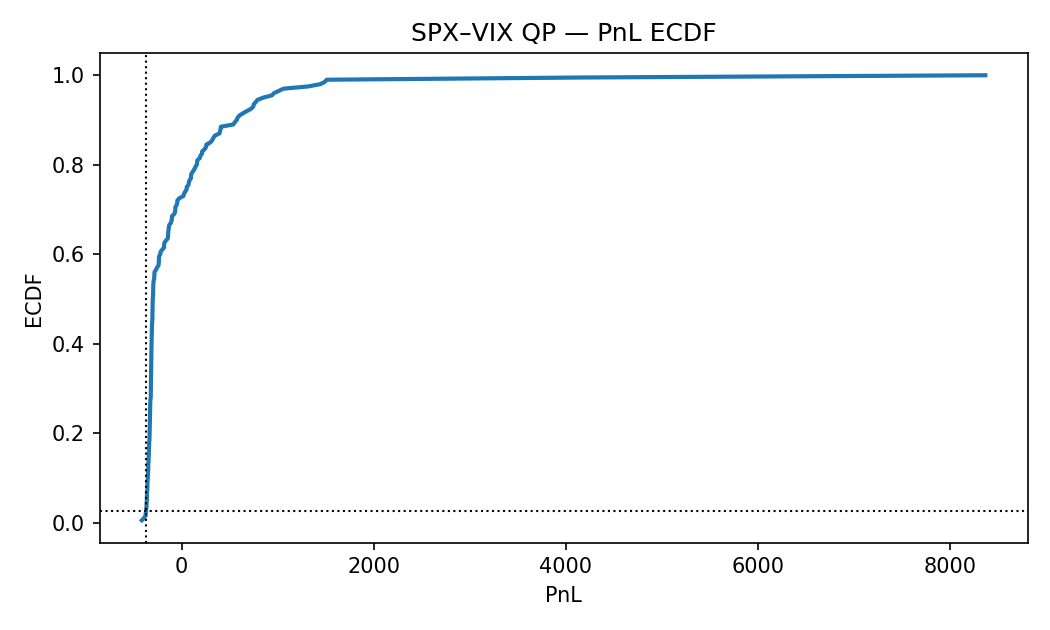}
    \caption{Baseline ECDF}
    \label{fig:baseline-ecdf}
  \end{subfigure}\hfill
  \begin{subfigure}[t]{0.48\linewidth}
    \centering
    \includegraphics[width=\linewidth]{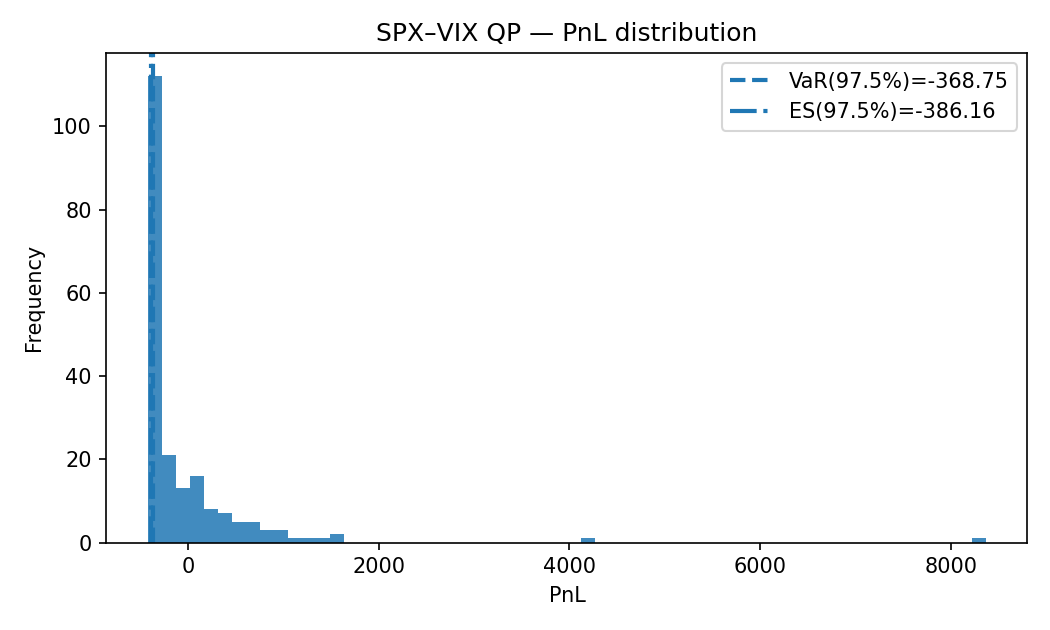}
    \caption{Baseline histogram (+VaR/ES)}
    \label{fig:baseline-hist}
  \end{subfigure}
  \vspace{-0.5em}
  \caption{\textbf{QP-only Baseline.} Unified samples; tail annotations correspond to $\alpha$ used at evaluation.}
  \label{fig:baseline-panels}
\end{figure}

\begin{figure}[t]
  \centering
  \begin{subfigure}[t]{0.48\linewidth}
    \centering
    \includegraphics[width=\linewidth]{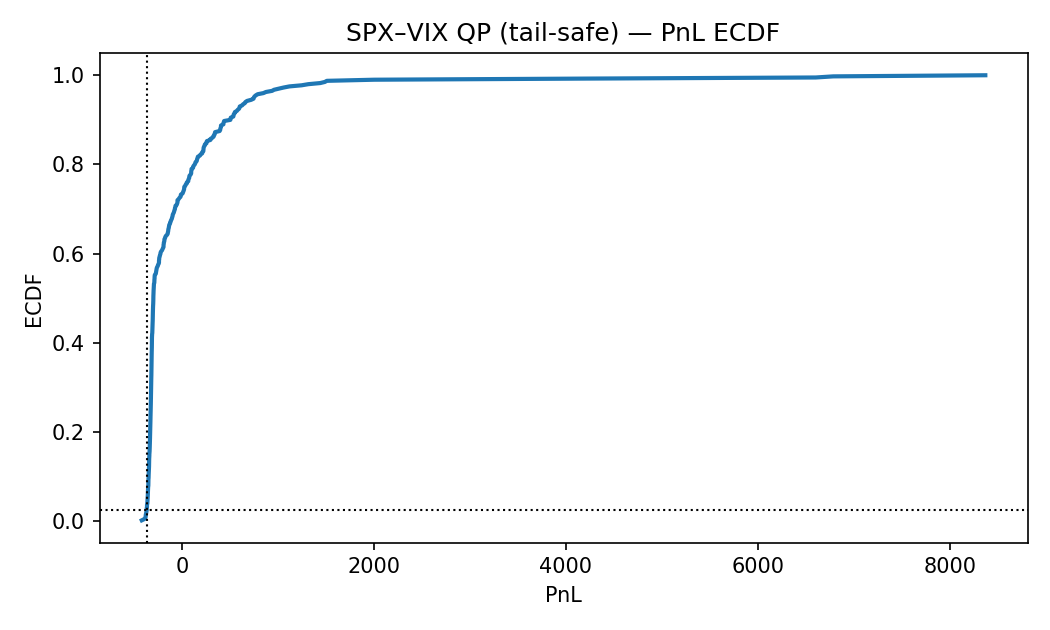}
    \caption{Tail-Safe ECDF}
    \label{fig:tailsafe-ecdf}
  \end{subfigure}\hfill
  \begin{subfigure}[t]{0.48\linewidth}
    \centering
    \includegraphics[width=\linewidth]{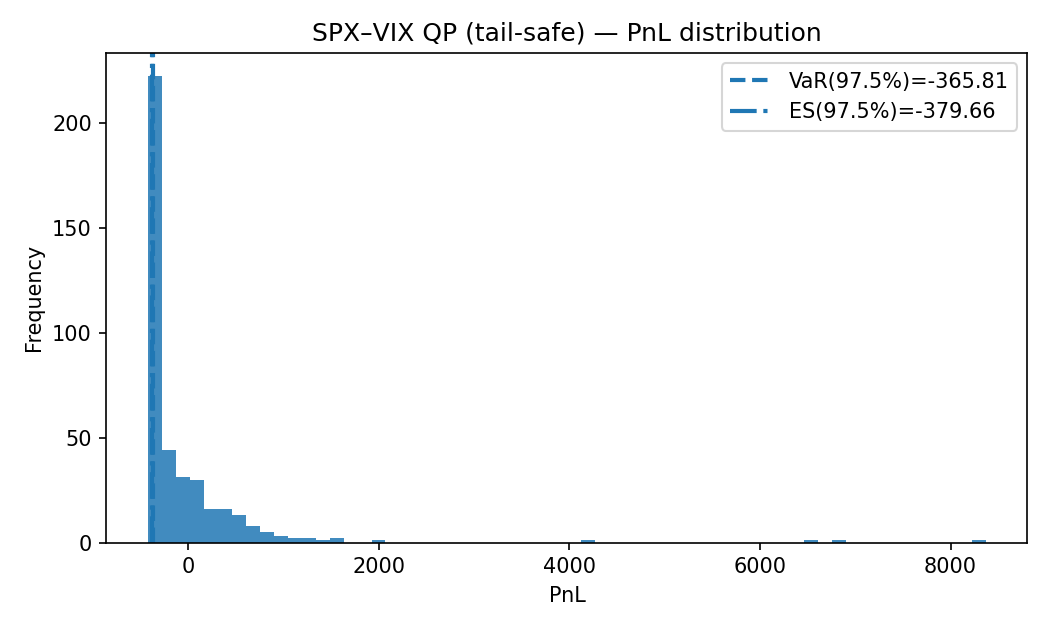}
    \caption{Tail-Safe histogram (+VaR/ES)}
    \label{fig:tailsafe-hist}
  \end{subfigure}
  \vspace{-0.5em}
  \caption{\textbf{Tail-Safe.} Left-tail improvement with central mass preserved}
  \label{fig:tailsafe-panels}
\end{figure}

\paragraph{Training dynamics.}
We observe small per-update policy KL (e.g., within a narrow band consistent with the PPO clip parameter), stable effective tail coverage $\widehat{w}\!\approx\!w_{\mathrm{target}}$ across the $\alpha$ schedule, and smooth critic losses; a single conservative spike in KL can occur when tightening $\alpha$, after which EMA referencing re-stabilizes updates (Appendix~C: logs and plots). 

\subsection{Ablations}
\label{sec:exp:ablations}

We ablate the key design components to attribute gains and understand failure modes:
\begin{enumerate}
  \item \textbf{No coverage controller} (uniform quantiles): degrades tail estimation and increases training variance; CVaR improvements diminish, especially at small $\alpha$.
  \item \textbf{No KL--DRO} ($\lambda_{\mathrm{KL}}{=}0$): increases policy drift; occasional regressions in tail metrics under OOD stress.
  \item \textbf{No tail guards in QP} (no NTB shrinkage / no sign gate): higher drawdowns near expiry and during correlation spikes; feasibility rate decreases.
  \item \textbf{Risk-neutral RL + CBF--QP} (no CVaR objective): preserves hard safety but does not systematically improve the left tail relative to \textbf{Tail-Safe}.
\end{enumerate}

\begin{table}[H]
\centering
\caption{Ablation matrix (qualitative summary). ``$\uparrow$''/``$\downarrow$'' indicate consistent directional changes vs.\ \textbf{Tail-Safe}; blank indicates neutral/mixed.}
\label{tab:ablations}
\small
\begin{tabular}{lcccccc}
\toprule
Variant & ES$_\alpha$ & VaR$_\alpha$ & Mean/Sharpe & Coverage $\widehat{w}$ & Feasibility & Solver time \\
\midrule
No coverage controller  & $\uparrow$ (worse) & $\uparrow$ & $\downarrow$ & unstable & $\sim$ & $\sim$ \\
No KL--DRO              & $\uparrow$ (worse) & $\uparrow$ & mixed & $\sim$ & $\sim$ & $\sim$ \\
No tail guards (QP)     & $\uparrow$ near expiry & $\uparrow$ & $\downarrow$ & $\sim$ & $\downarrow$ & $\sim$ \\
Risk-neutral RL + CBF--QP & $\uparrow$ vs.\ Tail-Safe & $\uparrow$ & mixed & $\sim$ & $\sim$ & $\sim$ \\
\bottomrule
\end{tabular}
\end{table}

\subsection{Safety Telemetry}
\label{sec:exp:telemetry}

\paragraph{Active constraints and tightness.}
We report the empirical distribution of \texttt{tightest\_id} across time/episodes (Appendix~C: bar plots). 
Near expiry, NTB and rate caps dominate; during volatility spikes, the sign-consistency gate activates more frequently, suppressing reactive trades.

\paragraph{Slack and feasibility.}
By construction, when the QP is feasible with $\zeta{=}0$, there are \emph{zero} hard-constraint violations (Theorem~\ref{thm:invariance}). 
We tabulate the feasibility rate and the distribution of \texttt{slack\_sum}; infeasibility incidents are rare and trigger a shaped penalty and logging (Appendix~C: incident table).

\paragraph{Solver latency.}
We summarize solver time per step (median/P95) and success rates for OSQP. 
Warm starts and modest conditioning of $H$ keep solve times stable across stress regimes.

\subsection{Takeaways}
\label{sec:exp:takeaways}

\textbf{Tail-Safe} delivers (i) \emph{hard safety with explanations} (active-set, tightness, rate/gate telemetry), (ii) \emph{tail shaping} (improved left-tail metrics as $\alpha$ tightens) with comparable central performance, and (iii) \emph{robustness under shift} via KL-regularized updates and tail-coverage stabilization. 
These properties hold across ID and OOD stress regimes in arbitrage-free, microstructure-aware markets.

\section{Explainability \& Governance}
\label{sec:explainability-governance}

This section operationalizes \textbf{Tail-Safe} for auditability and oversight. 
We (i) map solver and learner \emph{telemetry} to concrete \emph{governance workflows} (who inspects which signals and how thresholds trigger actions), 
(ii) instantiate common \emph{business rules} (leverage, liquidity, short-sale, drawdown) as white-box CBF constraints together with human-readable interception rationales, and 
(iii) define an incident taxonomy and escalation playbook consistent with financial model-risk guidance (e.g., SR~11-7, BCBS~239, NIST AI RMF)~\cite{FedSR11_7,BCBS239,NISTAI_RMF2023}.

\subsection{Telemetry-to-Governance Mapping}
\label{sec:explainability:telemetry}

The CBF--QP solver and the RL learner emit stepwise and aggregate telemetry (Sec.~\ref{sec:method:cbf}): 
\texttt{active\_set}, \texttt{tightest\_id}, \texttt{rate\_util}, \texttt{gate\_score}, \texttt{slack\_sum}, \texttt{solver\_status/time}, the per-update policy \texttt{KL\_step}, the effective tail coverage $\widehat{w}$, and the CVaR schedule $\alpha$. 
Table~\ref{tab:governance-map} translates these signals into dashboard tiles, thresholds, owners, and actions. 
Thresholds are conservative defaults; institution-specific limits can be codified in a policy registry with versioning (Appendix~C).

\begin{table}[H]
\centering
\caption{Telemetry $\rightarrow$ governance workflow mapping. ``Owner'' follows a RACI-style assignment: Trader (TR), Market Risk (MR), Model Risk Management (MRM), Compliance (CO), Operations (OPS), Internal Audit (IA).}
\label{tab:governance-map}
\small
\begin{tabular}{lllll}
\toprule
Signal & Dashboard tile & Threshold (example) & Owner & Action on breach \\
\midrule
\texttt{tightest\_id}           & Constraint frequency barplot & NTB or Rate $>$ 60\% for 1d      & MR $\rightarrow$ TR & Tighten bands; reschedule \\
\texttt{slack\_sum}             & Slack mass time series       & Any $>0$ (event)                 & MRM/CO              & Incident; penalty; freeze size \\
\texttt{solver\_status/time}    & P95 latency/solve status     & P95 $>$ 20ms or nonoptimal $>$0.1\% & OPS               & Warm-start; failover; add capacity \\
\texttt{rate\_util}             & Rate utilization histogram   & P95 $>0.9$ for 1h                & MR $\rightarrow$ TR & Cut $r_{\max}$; stagger hedges \\
\texttt{gate\_score}            & Gate pass-rate               & Pass-rate $<85\%$ intraday       & MRM                 & Audit; recal $\delta_{\mathrm{adv}}$ \\
\texttt{KL\_step}               & Policy KL per update         & P95 $>\beta^\star$               & MRM                 & $\uparrow\lambda_{\mathrm{KL}}$; slow schedule \\
$\widehat{w}$ vs.\ $w_{\mathrm{target}}$ & Tail coverage gauge     & $|\widehat{w}-w_{\mathrm{target}}|>0.02$ & MRM           & PID retune; $\uparrow$batch \\
$\alpha$ schedule               & CVaR level tracker           & Failed to tighten on time        & MRM                 & Hold; check variance \\
\bottomrule
\end{tabular}
\end{table}

\paragraph{Explanation artifacts.}
For each interception (i.e., when $u_t^{\mathrm{nom}} \neq u_t$), we log a structured \emph{Explanation Record}:
(i) IDs/names of the active constraints (from \texttt{active\_set}), 
(ii) the \emph{tightest} constraint and its slack/multiplier, 
(iii) the $H$-metric deviation $\|u_t{-}u_t^{\mathrm{nom}}\|_H$, 
(iv) human-readable rule names (e.g., ``Leverage limit''), and 
(v) a natural-language rationale template (below). 
Records are immutable and indexed for audit queries (Appendix~C.2: storage schema).

\subsection{Business Rules as CBF Constraints (with Explanations)}
\label{sec:explainability:rules}

We instantiate common policy limits as barrier functions $h_i(x)\!\ge\!0$; each one admits a plain-language explanation and a quantitative telemetry view.

\begin{table}[H]
\centering
\caption{Business rules $\rightarrow$ CBF instantiation and explanation.}
\label{tab:rules2cbf}
\small
\begin{tabular}{lll}
\toprule
Rule & CBF form (illustrative) & Explanation (human-readable) \\
\midrule
Leverage cap & $h_{\mathrm{lev}}(x)=L_{\max}-\mathrm{Lev}(x)$ & ``Leverage cannot exceed $L_{\max}$; trade reduced to keep $h_{\mathrm{lev}}\ge 0$.'' \\
Liquidity budget & $h_{\mathrm{liq}}(x,u)=V_{\mathrm{avail}}(x)-\|u\|_{\Lambda}$ & ``Order size limited by available depth/impact budget.'' \\
Short-sale limit & $h_{\mathrm{short}}(x)=Q_{\min}-q_{\mathrm{short}}(x)$ & ``Short inventory cannot breach $Q_{\min}$.'' \\
Drawdown guard & $h_{\mathrm{dd}}(x)=D_{\max}-\mathrm{DD}(x)$ & ``Cumulative drawdown kept below $D_{\max}$.'' \\
Rate cap & implicit via $\|u_t{-}u_{t-1}\|_2\le r_{\max}$ & ``Adjustment rate bounded by $r_{\max}$ to avoid whipsaw.'' \\
NTB (ellipsoid) & $e^\top M e\le b_{\max}$ & ``Exposure error confined within the no-trade band.'' \\
Sign gate & $g_{\mathrm{cons}}(x,u)\ge 0$ & ``Trade direction must align with signals beyond $\delta_{\mathrm{adv}}$.'' \\
\bottomrule
\end{tabular}
\end{table}

\paragraph{Illustrative explanations.}
For each rule, the solver’s KKT multipliers quantify \emph{how binding} the rule is; we surface them as ``tightness'' bars in the dashboard. 
For example, if \texttt{tightest\_id} corresponds to \emph{Rate cap}, the explanation highlights the capped step and suggests staggering or a temporary $r_{\max}$ reduction.

\subsection{Incident Taxonomy and Escalation Playbook}
\label{sec:explainability:incidents}

We classify events into three severities with actionable responses and owners.

\begin{table}[H]
\centering
\caption{Incident taxonomy and escalation.}
\label{tab:incidents}
\small
\begin{tabular}{lll l}
\toprule
Severity & Condition & Owner & Response \\
\midrule
S1 (soft intercept) & $u_t^{\mathrm{nom}}\neq u_t$, $\|\zeta\|_1=0$ & TR/MR & Log record; no halt; monitor tightest frequencies \\
S2 (hard intercept) & $\|\zeta\|_1>0$ or repeated rate saturation & MRM/CO & Apply penalty; partial freeze; review constraints; RCA within 1d \\
S3 (solver failure) & Non-optimal status or timeout & OPS/MRM & Failover; revert to baseline QP-only; RCA within 4h; postmortem \\
\bottomrule
\end{tabular}
\end{table}

\paragraph{Trigger logic (runtime).}
The following policies are enforced online and recorded for audit (pseudo-code; full governance rules in Appendix~C.3):
\begin{itemize}
  \item \textbf{Slack trigger}: if $\texttt{slack\_sum}>0$ for $k$ consecutive steps ($k\!\in\![1,3]$), (i) downscale action norm by $\eta\!<\!1$, (ii) raise alert to MRM, (iii) freeze size upon repeat.
  \item \textbf{Rate saturation trigger}: if $\texttt{rate\_util}>0.95$ at P95 over 10 minutes, tighten $r_{\max}$ by factor $\eta_r$ and stagger orders.
  \item \textbf{Gate trigger}: if gate pass-rate drops below $85\%$ intraday, increase $\delta_{\mathrm{adv}}$ and review signals; if unresolved, switch to QP-only baseline.
  \item \textbf{KL drift trigger}: if $\texttt{KL\_step}>\beta^\star$ at P95 for 5 updates, raise $\lambda_{\mathrm{KL}}$ and slow the $\alpha$ schedule.
\end{itemize}

\subsection{Audit Queries and Periodic Reviews}
\label{sec:explainability:audit}

In addition to runtime triggers, periodic reviews (weekly/monthly) support compliance and internal audit:
\begin{enumerate}
  \item \textbf{Constraint mix}: distribution of \texttt{tightest\_id} by regime (ID/OOD, expiry buckets), highlighting persistent bottlenecks.
  \item \textbf{Intercept cost}: average $H$-norm deviation $\|u{-}u^{\mathrm{nom}}\|_H$ and its contribution to P\&L shortfall (Appendix~C.4 decomposition).
  \item \textbf{Feasibility trend}: time series of feasibility rate and slack mass; root-cause analysis for S2/S3 spikes.
  \item \textbf{Tail control}: realized $\operatorname{ES}_\alpha$ vs.\ scheduled $\alpha$; coverage tracking $(\widehat{w},w_{\mathrm{target}})$.
  \item \textbf{Change control}: diffs in policy parameters and constraint registries; sign-offs by MRM/CO; SR~11-7 documentation hooks~\cite{FedSR11_7}.
\end{enumerate}

\paragraph{Human-readable \& machine-readable parity.}
Every interception has (i) a natural-language rationale (template above) and (ii) a structured record (IDs, multipliers, thresholds, margins). 
This parity serves both reviewers and automated conformance checks (BCBS~239 data lineage and auditability requirements)~\cite{BCBS239}.

\subsection{Risk, Compliance, and Scope Notes}
\label{sec:explainability:scope}

\emph{Scope.} Tail-Safe enforces selected constraints by construction and exposes telemetry for others. 
Constraints whose surrogates are only approximate (e.g., complex liquidity or borrow availability) are flagged as ``monitor-only'' with thresholds and escalation but not guaranteed by the CBF invariance (Sec.~\ref{sec:theory}). 

\emph{Residual risk.} When the QP is infeasible (rare), we document the fallback (penalty + baseline), the magnitude/duration, and the mitigation timeline, aligning with MRM incident handling~\cite{FedSR11_7,NISTAI_RMF2023}. 
Broader regulatory mapping (e.g., model lifecycle, validation independence, change control) is included in Appendix~C (Governance Checklist).

\section{Limitations and Scope}
\label{sec:limitations}

We summarize the principal limitations of \textbf{Tail-Safe} to clarify the scope of our claims. These caveats are intended to be candid and actionable without diminishing the theoretical contributions in Sections~\ref{sec:method}--\ref{sec:theory}. Where appropriate, we outline concrete paths to address each limitation in future work.

\paragraph{Synthetic market only.}
All experiments (Sec.~\ref{sec:experiments}) are conducted in an \emph{arbitrage-free but synthetic} environment built from SSVI~$\rightarrow$~Dupire~$\rightarrow$~VIX (\S\ref{sec:ssvi}--\ref{sec:vix}). 
While this stack ensures static no-arbitrage and inherits stylized skew/term-structure, it does not capture the full richness of historical dynamics (macro news, regime shifts, structural breaks, jumps, cross-sectional co-movement, borrow frictions, corporate actions). 
\emph{Scope:} Our empirical claims (tail-shaping, feasibility rates, telemetry behavior) pertain to this controlled setting. 
\emph{Next steps:} historical \emph{replay} against top-of-book/L1 or full LOB data; \emph{semi-synthetic} playback injecting real returns with simulated fills; and \emph{conditional stress} using realized event calendars.

\paragraph{Simplified execution and microstructure.}
The execution adapter abstracts to spread, linear temporary impact, and a transient resilience kernel (Sec.~\ref{sec:execution}). 
It omits queue-reactive dynamics, order priority, asymmetric information/adverse selection, hidden liquidity, fee/rebate tiers, tick-size effects, and venue fragmentation. 
\emph{Scope:} Results on cost and latency are indicative, not definitive, for live markets. 
\emph{Next steps:} queue-based impact models, order-cancels-replace loops, latency/jitter modeling, cross-venue routing, and \emph{empirical} calibration to venue-level impact curves.

\paragraph{Approximation in safety-layer modeling.}
The discrete-time CBF constraints rely on local models $x_{t+1}=f(x_t)+g(x_t)u_t$ and Lipschitz/mismatch bounds (Assumption~\ref{assump:dynamics}). 
Barrier functions encode business rules via convex surrogates and linearizations in $u$ (Assumption~\ref{assump:convex}). 
\emph{Scope:} The \emph{forward-invariance guarantee} (Theorem~\ref{thm:invariance}) holds when the QP is feasible with zero slack (\,$\zeta{=}0$) and the margin condition is satisfied. 
When the QP is infeasible (rare in our synthetic tests), the layer reverts to penalties and logging; it does not \emph{guarantee} satisfaction. 
\emph{Next steps:} robust/tube CBFs, disturbance observers, set-valued dynamics, and feasibility recovery strategies; formal analysis of linearization error budgets (Appendix~A.6 road map).

\paragraph{Estimator variance and controller tuning.}
The temperature-tilted CVaR estimator is self-normalized and subject to variance/bias trade-offs controlled by $(T,\gamma_{\mathrm{tail}})$ and the coverage controller (Sec.~\ref{sec:method:coverage}). 
\emph{Scope:} The concentration result (Theorem~\ref{thm:cvar-concentration}) holds under bounded losses and weights, and for an \emph{effective} tail mass close to the target; misspecification of these conditions can inflate variance. 
\emph{Next steps:} adaptive batch-sizing, control-theoretic anti-windup, and doubly-robust estimators for tail integrals.

\paragraph{Trust-region and robustness assumptions.}
The DRO bound (Theorem~\ref{thm:kldro}) relies on occupancy coupling to relate pathwise and per-state KL; 
the CVaR trust-region result (Theorem~\ref{thm:trpo-cvar}) uses Lipschitz smoothness and Pinsker-type bounds to control distribution shift. 
\emph{Scope:} These are standard but \emph{local} assumptions; violations (e.g., abrupt policy changes, heavy-tailed losses) can weaken constants or rates. 
\emph{Next steps:} stronger coupling via mixing coefficients or spectral gaps; nonasymptotic pathwise bounds; heavy-tail robustification.

\paragraph{Metric coverage.}
We optimize and report CVaR/ES at a given $\alpha$ (\S\ref{sec:risk}); other tail functionals (extreme quantiles, expectiles, drawdown risk, spectral measures) are not directly optimized. 
\emph{Scope:} Improvements outside the targeted $\alpha$ are empirical rather than guaranteed. 
\emph{Next steps:} multi-level or spectral risk objectives; path-dependent risk control (drawdown CBFs) with proofs paralleling Theorem~\ref{thm:invariance}.

\paragraph{Generalization across books and instruments.}
We focus on a stylized SPX--VIX hedging book. 
\emph{Scope:} Constraints and signals (NTB, sign gate) are finance-specific but not yet validated for rates/credit/commodities or multi-currency portfolios. 
\emph{Next steps:} multi-asset extension with cross-gamma/vega and inventory coupling; borrow/rehypothecation constraints; FX basis and curve risk.

\paragraph{Governance and compliance boundaries.}
The telemetry-to-governance mapping (Sec.~\ref{sec:explainability-governance}) aligns with common guidance, but it is not a substitute for institution-specific model governance (validation independence, challenger models, change control). 
\emph{Scope:} We provide an \emph{operational} starting point (dashboards, triggers, audit records) rather than a complete compliance program. 
\emph{Next steps:} integration with enterprise MRM workflows (SR~11-7/BCBS~239/NIST AI RMF), automated conformance checks, and periodic model risk reviews.

\paragraph{Ethical and adversarial considerations.}
We do not model adversarial behavior (probing, spoofing, latency games) against the agent or the execution venue. 
\emph{Scope:} Safety guarantees are not security guarantees. 
\emph{Next steps:} adversarial stress testing, red-team evaluations, and anomaly detection hooks tied to telemetry.

\paragraph{Computational constraints.}
Real-time feasibility depends on solver latency, model size, and hardware. Our OSQP-based implementation meets timing in the synthetic setup, but live deployments may face tighter budgets. 
\emph{Scope:} The results demonstrate feasibility \emph{in silico}. 
\emph{Next steps:} problem sparsity exploitation, warm-start policies, batching, and hardware acceleration.

\paragraph{Summary of claims.}
(i) \emph{Hard safety} is guaranteed by construction \emph{only} when the QP is feasible with zero slack and under the margin conditions (Thm.~\ref{thm:invariance}); 
(ii) \emph{Tail improvements} are demonstrated empirically in arbitrage-free synthetic markets and supported by estimator concentration (Thm.~\ref{thm:cvar-concentration}) and trust-region reasoning (Thm.~\ref{thm:trpo-cvar}); 
(iii) \emph{Explainability} follows from the CBF--QP’s active-set/dual telemetry and the projection interpretation (Prop.~\ref{prop:projection}). 
Real-data backtesting and broader market coverage are deliberately left to future work to preserve a clean separation between theory and synthetic evaluation.

\section{Related Work}
\label{sec:related}

We review connections to prior work on (i) deep hedging and RL in finance, (ii) risk-sensitive/CVaR reinforcement learning, (iii) safety filters and control-barrier-function (CBF) methods for safe RL, (iv) distributionally robust RL (DRO-RL), and (v) explainability and governance for financial ML. Our approach differs by combining \emph{distributional, CVaR-optimized learning} with a \emph{white-box CBF--QP safety layer} that provides \emph{stepwise audit telemetry} and by establishing \emph{theoretical guarantees} (Sec.~\ref{sec:theory}) that link KL-regularized policy updates to a KL--DRO upper bound and a CVaR trust-region improvement inequality, while proving robust forward invariance and feasibility persistence for the safety layer.

\paragraph{Deep hedging and RL in finance.}
Deep hedging learns nonparametric hedge policies under frictions and liquidity constraints~\cite{Buehler2019DeepHedging}, and recent surveys catalogue advances in trading/hedging RL~\cite{Hambly2023RLFinance,Sun2023RLTrading,Pippas2025CSUR}. 
Other lines integrate microstructure and execution---from Almgren--Chriss and Obizhaeva--Wang impact models~\cite{AlmgrenChriss2001,ObizhaevaWang2013} to agent-based LOB simulators (ABIDES)~\cite{Byrd2020ABIDES,Amrouni2021ABIDESGym} and market-microstructure monographs~\cite{CarteaJaimungal2015}. 
\textbf{Tail-Safe} differs in three respects: (i) it \emph{optimizes a tail risk} (CVaR) rather than mean performance, (ii) it enforces \emph{hard state-wise constraints by construction} using a CBF--QP filter with telemetry, and (iii) it provides \emph{formal guarantees} on safety invariance, DRO conservatism via KL, and trust-region CVaR improvement (Theorems~\ref{thm:invariance}, \ref{thm:kldro}, \ref{thm:trpo-cvar}).

\paragraph{Risk-sensitive RL and CVaR objectives.}
Risk-sensitive RL with coherent risk measures (e.g., CVaR) has been studied via policy gradients and actor--critic methods~\cite{Chow2014CVaRMDP,Tamar2015CoherentRisk,Chow2018VaRCVaR}, with distributional critics (IQN) offering flexible quantile modeling~\cite{Bellemare2017Distributional,Dabney2018IQN}. 
Recent works explore scalable or constrained CVaR optimization and robust variants~\cite{Shi2024JMLR,Lu2024TVDRL,SundharRamesh2024DRMBRL}. 
Our contribution is orthogonal: we introduce a \emph{coverage-controlled} quantile sampling scheme that directly stabilizes the \emph{estimation error} of CVaR at small $\alpha$ (Theorem~\ref{thm:cvar-concentration}), and couple it with a \emph{white-box} safety layer; we also provide a \emph{CVaR trust-region inequality} under per-state KL limits (Theorem~\ref{thm:trpo-cvar}).

\paragraph{Safe RL: shielding, projection, and CBF-based filters.}
Safe RL surveys cover constraint handling via penalties, Lagrangians, shielding, and model-based fallback~\cite{GarciaFernandez2015,Zhao2023StatewiseSafeRL}. 
Shielding for discrete MDPs uses a precomputed safety automaton to override unsafe actions~\cite{Alshiekh2018Shielding}; action projection in continuous control enforces linearized safety via a QP layer~\cite{Dalal2018Projection}. 
In control theory, CBF--QP methods enforce forward invariance of safe sets at each step~\cite{Ames2017CBF_TAC,Ames2019CBF_Tutorial}, with robust/high-order/differentiable extensions~\cite{Jankovic2018RobustCBF,Ma2022DiffCBF,Yang2022DiffCBF,Garg2024CBFAdvances}. 
Differentiable optimization layers (QP/convex) have also been embedded in deep models~\cite{AmosKolter2017OptNet,Agrawal2019DiffOpt}. 
\textbf{Tail-Safe} leverages \emph{discrete-time} CBF constraints specialized to finance (ellipsoidal NTB, box/rate, sign-consistency gate), \emph{proves} robust forward invariance under bounded mismatch (Theorem~\ref{thm:invariance}) and minimal-deviation projection (Proposition~\ref{prop:projection}), and \emph{surfaces KKT/active-set telemetry} to enable audit workflows (Sec.~\ref{sec:explainability-governance}). 
Compared to generic projection/shielding, our layer is domain-specific (exposure- and microstructure-aware) and backed by an incident-handling governance mapping.

\paragraph{Distributionally robust RL (DRO-RL).}
Robust MDPs with uncertainty sets date back to classical analyses~\cite{Iyengar2005RobustDP,NilimElGhaoui2005}, with modern variants using $f$-divergences or Wasserstein ambiguity sets for policy evaluation and learning~\cite{Derman2018DRL,Lu2024TVDRL,SundharRamesh2024DRMBRL}. 
In parallel, KL-regularized policy optimization (TRPO/PPO) has a long-standing trust-region interpretation~\cite{Schulman2015TRPO,Schulman2017PPO} and has been adapted to constrained settings~\cite{Achiam2017CPO}. 
We bridge these threads by showing that per-state KL penalties control a \emph{KL--DRO upper bound} on the CVaR surrogate (Theorem~\ref{thm:kldro}) and yield a \emph{CVaR trust-region improvement} bound (Theorem~\ref{thm:trpo-cvar})—to our knowledge, such a pairing of DRO interpretation and CVaR performance-difference analysis has not been stated explicitly for the finance setting with a safety filter.

\paragraph{Explainability and governance in financial ML.}
Interpretability surveys argue for transparent mechanisms over post-hoc explanations in high-stakes domains~\cite{DoshiVelezKim2017,Guidotti2018Survey,Rudin2019StopExplaining}. 
Financial firms operate under model-risk governance such as SR~11-7 and BCBS~239, requiring documentation, monitoring, and auditable decision records. 
Our telemetry-to-governance mapping (Sec.~\ref{sec:explainability-governance}) uses \emph{white-box} CBF--QP artifacts (active sets, multipliers, slack, rate utilization, gate scores) to produce \emph{structured} and \emph{human-readable} rationale for each interception, aligning methodologically with those expectations while preserving by-construction safety guarantees.

\paragraph{Summary of differences.}
Relative to prior deep hedging and RL-in-finance work, \textbf{Tail-Safe} provides (i) \emph{formal tail-risk learning guarantees} (coverage-controlled CVaR estimation and trust-region inequality), (ii) a \emph{finance-specialized, auditable} CBF--QP safety layer with invariance and projection properties, and (iii) a \emph{DRO-consistent} view of KL-regularized updates. 
This combination targets \emph{deployable} hedging: robust to distribution shift, safe by construction, and explainable to risk managers and auditors.
\section{Conclusion}
\label{sec:conclusion}

We presented \textbf{Tail-Safe}, a deployability-oriented framework for hedging under market frictions that unifies \emph{distributional, risk-sensitive learning} with a \emph{white-box} safety layer. 
On the learning side, we optimize tail risk via IQN--CVaR--PPO and stabilize small-$\alpha$ training through a \emph{Tail-Coverage Controller} that regulates quantile sampling with temperature and tail-boost. 
On the safety side, we enforce state-wise constraints by construction with a finance-specialized CBF--QP filter---combining discrete-time CBF inequalities, an ellipsoidal no-trade band (NTB), box/rate limits, and a sign-consistency gate—and expose \emph{audit-ready telemetry} (active sets, tightness, slack, rate utilization, gate scores, solver status).

\paragraph{Theoretical guarantees.}
We established a suite of results that elevate Tail-Safe from a pragmatic recipe to a principled method: 
(i) \emph{robust forward invariance} of the safety set under bounded model mismatch (Theorem~\ref{thm:invariance}); 
(ii) a \emph{minimal-deviation} projection interpretation of the safety QP (Proposition~\ref{prop:projection}); 
(iii) a \emph{KL--DRO upper bound} connecting per-state KL regularization to distributionally robust CVaR control (Theorem~\ref{thm:kldro}); 
(iv) \emph{concentration and sample-complexity} of the temperature-tilted CVaR estimator with explicit coverage-mismatch terms (Theorem~\ref{thm:cvar-concentration}); 
(v) a \emph{CVaR trust-region improvement} inequality under KL limits (Theorem~\ref{thm:trpo-cvar}); 
(vi) \emph{feasibility persistence} guarantees for expiry-aware NTB shrinkage and rate tightening (Theorem~\ref{thm:feas-persist}); and 
(vii) \emph{negative-advantage suppression} induced by the sign-consistency gate (Proposition~\ref{prop:gate}). 
Formal proofs are deferred to Appendix~A.

\paragraph{Empirical evidence and governance.}
In arbitrage-free, microstructure-aware synthetic markets (SSVI$\!\rightarrow$Dupire$\!\rightarrow$VIX with ABIDES/MockLOB execution), Tail-Safe achieves \emph{left-tail improvements} while preserving central performance, with \emph{zero} hard-constraint violations whenever the QP is feasible with zero slack. 
Results are reported with unified sample sizes, multiple seeds, and paired bootstrap confidence intervals.
Beyond performance metrics, we demonstrate \emph{operational explainability}: solver/learner telemetry is mapped to governance dashboards and incident playbooks (Sec.~\ref{sec:explainability-governance}), providing structured, human-readable rationales for each interception alongside machine-verifiable records.

\paragraph{Scope and limitations.}
Our empirical claims are confined to synthetic but finance-grounded settings; the execution layer abstracts queue-reactive and venue-specific effects; and some guarantees require feasibility and margin conditions (Sec.~\ref{sec:limitations}). 
These boundaries are deliberate: they isolate the methodological core and create a clean landing zone for subsequent validation.

\paragraph{Outlook.}
We see three promising thrusts for future work. 
\emph{(A) Real-data validation:} historical replays with top-of-book or full LOB feeds; semi-synthetic playback that couples realized returns with simulated fills; and event-conditioned stress studies. 
\emph{(B) Richer safety and robustness:} queue-reactive impact, borrow/rehypothecation constraints, drawdown and path-dependent CBFs, heavy-tail robust estimators, and tighter pathwise DRO couplings. 
\emph{(C) Broader portfolios and lifecycles:} multi-asset books with cross-gamma/vega limits, cross-currency hedging, offline/online hybrids, and deeper integration with enterprise model-risk workflows (change control, challenger models, periodic validation). 

\paragraph{Final remark.}
Tail-Safe aims to narrow the gap between modern RL and institutional requirements by offering \emph{tail-aware learning}, \emph{by-construction safety}, and \emph{auditability} in one coherent framework.
We hope the blend of theory, engineering, and governance presented here provides a reproducible baseline—and a stepping stone toward trustworthy, regulation-ready AI for financial risk management.
\appendix
\section{Proofs for Section~\ref{sec:theory}}
\label{app:proofs}

This appendix provides complete proofs for Theorems~\ref{thm:invariance}, \ref{thm:kldro}, \ref{thm:cvar-concentration}, \ref{thm:trpo-cvar}, \ref{thm:feas-persist} and Propositions~\ref{prop:projection}, \ref{prop:gate}. 
We use the notation and assumptions introduced in Sections~\ref{sec:prelim}--\ref{sec:theory}. 
For a positive definite matrix $H\succ 0$, we write $\|v\|_{H} := \sqrt{v^\top H v}$ and $\langle a,b\rangle_H := a^\top H b$.

\subsection{Proof of Theorem~\ref{thm:invariance} (robust forward invariance)}
\label{app:proof-invariance}

\begin{lemma}[Lipschitz pushforward under bounded disturbance]
\label{lem:lipschitz-push}
Let $h:\mathbb{R}^d\!\to\!\mathbb{R}$ be $L$-Lipschitz. If $\|w\|\le \bar w$, then for any $z\in\mathbb{R}^d$,
$
h(z+w) \ge h(z) - L\bar w.
$
\end{lemma}
\begin{proof}
By Lipschitz continuity, $|h(z+w)-h(z)|\le L\|w\|\le L\bar w$, hence $h(z+w)\ge h(z) - L\bar w$.
\end{proof}

\begin{proof}[Proof of Theorem~\ref{thm:invariance}]
Fix $t$ and an index $i$. By Assumption~\ref{assump:dynamics}, $x_{t+1}=f(x_t)+g(x_t)u_t+w_t$ with $\|w_t\|\le \bar w$, and by Assumption~\ref{assump:cbf} we have
\begin{equation}
\label{eq:margin-app}
h_i\!\big(f(x_t)+g(x_t)u_t\big) - (1-\kappa_i\Delta t)h_i(x_t) \ge \varepsilon_i \ge L_i \bar w.
\end{equation}
Applying Lemma~\ref{lem:lipschitz-push} with $z=f(x_t)+g(x_t)u_t$ and $w=w_t$ gives
\[
h_i(x_{t+1}) \;\ge\; h_i\!\big(f(x_t)+g(x_t)u_t\big) - L_i\bar w 
\;\stackrel{\eqref{eq:margin-app}}{\ge}\; (1-\kappa_i\Delta t)h_i(x_t) + \varepsilon_i - L_i\bar w \;\ge\; (1-\kappa_i\Delta t)h_i(x_t).
\]
If $h_i(x_t)\ge 0$, then $h_i(x_{t+1})\ge 0$. By induction over $t$, the set $\mathcal{C}=\{x:\,h_i(x)\ge 0,\,\forall i\}$ is forward invariant.
\end{proof}

\subsection{Proof of Proposition~\ref{prop:projection} (minimal-deviation projection)}
\label{app:proof-projection}

\begin{lemma}[Strict convexity and uniqueness]
\label{lem:unique}
Under Assumption~\ref{assump:convex}, if $H\succ 0$ and the feasible set $\mathcal{S}(x_t)$ is nonempty, closed, and convex, then the QP~\eqref{eq:qp-safety} with $\zeta=0$ has a unique minimizer.
\end{lemma}
\begin{proof}
The objective $u\mapsto \frac12\|u-u^{\mathrm{nom}}\|_H^2 + c^\top u$ is strictly convex due to $H\succ 0$. A strictly convex function over a nonempty, closed, convex set attains its minimum at a unique point.
\end{proof}

\begin{proof}[Proof of Proposition~\ref{prop:projection}]
Completing the square yields
\[
\frac12(u-u^{\mathrm{nom}})^\top H (u-u^{\mathrm{nom}}) + c^\top u
= \frac12 \big\|u - (u^{\mathrm{nom}}-H^{-1}c)\big\|_H^2 + \mathrm{const},
\]
where the constant does not depend on $u$. Hence minimizing the QP with $\zeta=0$ over $\mathcal{S}(x_t)$ is equivalent to projecting $u^{\mathrm{nom}}-H^{-1}c$ onto $\mathcal{S}(x_t)$ in the $H$-norm. Uniqueness follows from Lemma~\ref{lem:unique}.
\end{proof}

\subsection{Proof of Theorem~\ref{thm:kldro} (KL--DRO upper bound)}
\label{app:proof-kldro}

We recall the Donsker--Varadhan variational inequality (e.g., \cite{DonskerVaradhan1975,CsiszarKorner2011}): for any measurable $\varphi$ and probability measures $Q,P$ with $Q\ll P$,
\begin{equation}
\label{eq:dv}
\mathbb{E}_{Q}[\varphi] \;\le\; \frac{1}{\eta}\Big(\mathrm{KL}(Q\|P) + \log \mathbb{E}_P[e^{\eta \varphi}]\Big),\quad \forall\,\eta>0.
\end{equation}

\begin{proof}[Proof of Theorem~\ref{thm:kldro}]
By Assumption~\ref{assump:cvar}, for any $t\in\mathbb{R}$,
\[
\mathrm{CVaR}_\alpha(L) \;=\; \min_{t}\, \Big\{ t + \frac{1}{\alpha}\mathbb{E}_Q[(L-t)_+]\Big\}
\;\le\; t + \frac{1}{\alpha}\mathbb{E}_Q[\phi_t(L)],
\]
where $\phi_t(z)=(z-t)_+$. Taking the supremum over $Q\in\mathcal{Q}_\rho=\{Q:\mathrm{KL}(Q\|P)\le \rho\}$ and applying \eqref{eq:dv} to $\varphi=\phi_t(L)$ gives
\[
\sup_{Q\in\mathcal{Q}_\rho}\mathrm{CVaR}_\alpha(L)
\;\le\; \min_t \left\{ t + \frac{1}{\alpha} \sup_{Q\in\mathcal{Q}_\rho}\mathbb{E}_Q[\phi_t(L)] \right\}
\;\le\; \min_t \left\{ t + \frac{1}{\alpha\eta}\Big(\rho + \log \mathbb{E}_P[e^{\eta \phi_t(L)}]\Big) \right\}.
\]
Under Assumption~\ref{assump:kl}, the pathwise radius satisfies $\rho \le C_{\mathrm{occ}}\beta$, thus yielding the second claim. This shows that penalizing the per-state KL to keep $\beta$ small controls a KL--DRO upper bound on the CVaR surrogate.
\end{proof}

\subsection{Proof of Theorem~\ref{thm:cvar-concentration} (temperature-tilted CVaR concentration)}
\label{app:proof-cvar-concentration}

We treat $\widehat{\mathrm{CVaR}}_\alpha$ as a self-normalized importance sampling (SNIS) ratio estimator over the \emph{tail block} $\{\tau\le \alpha\}$. Let $I_k=\mathbf{1}\{\tau_k\le \alpha\}$ and define
\[
X_k \;=\; w_k I_k L^{(\tau_k)},\qquad Y_k \;=\; w_k I_k,\qquad 
\widehat{\mathrm{CVaR}}_\alpha \;=\; \frac{\sum_{k=1}^K X_k}{\sum_{k=1}^K Y_k}.
\]
We use that $|L|\le B$ and $p_T(\tau)\ge p_{\min}>0$ (Assumption~\ref{assump:is}), hence $w_k \le 1/p_{\min}=:W_{\max}$ and $|X_k|\le B W_{\max}$, $0\le Y_k \le W_{\max}$.
\begin{lemma}[Concentration for SNIS ratios with bounded weights]
\label{lem:snis}
Let $(X_k,Y_k)_{k=1}^K$ be i.i.d.\ pairs with $|X_k|\le a$, $0\le Y_k\le b$, and denote
$\mu_X=\mathbb{E}[X_1]$, $\mu_Y=\mathbb{E}[Y_1]>0$. Then for any $\delta\in(0,1)$, with probability at least $1-\delta$,
\[
\left|\frac{\sum_{k=1}^K X_k}{\sum_{k=1}^K Y_k} - \frac{\mu_X}{\mu_Y} \right|
\;\le\; 
\frac{2a}{\mu_Y}\sqrt{\frac{\log(4/\delta)}{2K}} \;+\;
\frac{2b|\mu_X|}{\mu_Y^2}\sqrt{\frac{\log(4/\delta)}{2K}}.
\]
\end{lemma}

\begin{proof}
Write $S_X:=\sum_{k=1}^K X_k$, $S_Y:=\sum_{k=1}^K Y_k$. Note that
\[
\frac{S_X}{S_Y}-\frac{\mu_X}{\mu_Y}
=\frac{S_X\mu_Y-\mu_X S_Y}{S_Y\mu_Y}
=\frac{\mu_Y(S_X-K\mu_X)-\mu_X(S_Y-K\mu_Y)}{S_Y\mu_Y}.
\]
Hence on any event where $S_Y>0$,
\begin{equation}\label{eq:ratio-decomp}
\left|\frac{S_X}{S_Y}-\frac{\mu_X}{\mu_Y}\right|
\le
\frac{\mu_Y\,|S_X-K\mu_X|+|\mu_X|\,|S_Y-K\mu_Y|}{S_Y\,\mu_Y}.
\end{equation}
We control numerator and denominator separately.

\smallskip\noindent\textbf{Step 1: Hoeffding bounds for $S_X$ and $S_Y$.}
Since $X_k\in[-a,a]$, Hoeffding's inequality yields, for any $t_X>0$,
\[
\mathbb{P}\!\left(\left|\frac{1}{K}S_X-\mu_X\right|\ge t_X\right)\le 2\exp\!\left(-\frac{2K t_X^2}{(2a)^2}\right)
=2\exp\!\left(-\frac{K t_X^2}{2a^2}\right).
\]
Similarly, $Y_k\in[0,b]$ implies, for any $t_Y>0$,
\[
\mathbb{P}\!\left(\left|\frac{1}{K}S_Y-\mu_Y\right|\ge t_Y\right)\le 2\exp\!\left(-\frac{2K t_Y^2}{b^2}\right).
\]
Choose
\[
t_X \;=\; a\sqrt{\frac{2\log(4/\delta)}{K}},\qquad
t_Y \;=\; \frac{b}{\sqrt{2K}}\sqrt{\log(4/\delta)}.
\]
Then, by a union bound, with probability at least $1-\delta$ we have simultaneously
\begin{equation}\label{eq:event-E}
\mathcal{E}:\quad
|S_X-K\mu_X|\le K t_X,\qquad |S_Y-K\mu_Y|\le K t_Y.
\end{equation}

\smallskip\noindent\textbf{Step 2: keep the denominator away from zero.}
On $\mathcal{E}$ we have the deterministic lower bound
\[
S_Y \;\ge\; K(\mu_Y-t_Y).
\]
If $t_Y\le \mu_Y/2$ (which is a mild requirement for $K$; see below), then $S_Y\ge K\mu_Y/2$. Even without this simplification, from~\eqref{eq:ratio-decomp} and~\eqref{eq:event-E} we obtain on $\mathcal{E}$:
\begin{align*}
\left|\frac{S_X}{S_Y}-\frac{\mu_X}{\mu_Y}\right|
&\le \frac{\mu_Y\,(K t_X)+|\mu_X|\,(K t_Y)}{K(\mu_Y-t_Y)\,\mu_Y}
\\
&= \frac{t_X}{\mu_Y-t_Y}\;+\; \frac{|\mu_X|}{\mu_Y}\cdot\frac{t_Y}{\mu_Y-t_Y}.
\end{align*}
If $t_Y\le \mu_Y/2$, then $\mu_Y-t_Y\ge \mu_Y/2$, hence
\[
\left|\frac{S_X}{S_Y}-\frac{\mu_X}{\mu_Y}\right|
\;\le\; \frac{2t_X}{\mu_Y} \;+\; \frac{2|\mu_X|}{\mu_Y^2}\,t_Y.
\]
Finally substitute our $t_X,t_Y$ choices to get
\[
\left|\frac{S_X}{S_Y}-\frac{\mu_X}{\mu_Y}\right|
\;\le\;
\frac{2a}{\mu_Y}\sqrt{\frac{\log(4/\delta)}{2K}}
\;+\;
\frac{2b|\mu_X|}{\mu_Y^2}\sqrt{\frac{\log(4/\delta)}{2K}},
\]
which is the claimed bound.

\smallskip
\noindent\textbf{Remark on the condition $t_Y\le \mu_Y/2$.}
Because $t_Y=\frac{b}{\sqrt{2K}}\sqrt{\log(4/\delta)}$, a sufficient condition is
$K \ge \frac{2b^2}{\mu_Y^2}\log\!\frac{4}{\delta}$.
If $K$ is smaller, we can keep the more general denominator factor $\mu_Y-t_Y$ in the bound, i.e.,
\[
\left|\frac{S_X}{S_Y}-\frac{\mu_X}{\mu_Y}\right|
\le \frac{t_X}{\mu_Y-t_Y} + \frac{|\mu_X|}{\mu_Y}\cdot\frac{t_Y}{\mu_Y-t_Y},
\]
which reduces to the displayed result once $\mu_Y-t_Y\ge \mu_Y/2$.
\end{proof}

\begin{lemma}[Tail-mass and divergence control]
\label{lem:alphaeff}
Let $\tau\sim p_T$ on $[0,1]$ and $I=\mathbf{1}\{\tau\le \alpha\}$, $\alpha\in(0,1)$. Define the \emph{effective tail mass} $\alpha_{\mathrm{eff}}:=\mathbb{E}[I]=\int_0^\alpha p_T(\tau)\,d\tau$. Then:

\begin{enumerate}
\item[(i)] For \emph{unnormalized} importance weights $w(\tau)=1/p_T(\tau)$, one has
\[
\mu_Y \;:=\; \mathbb{E}[w(\tau)\,I]
= \int_0^\alpha \frac{1}{p_T(\tau)}\,p_T(\tau)\,d\tau
= \alpha.
\]
\item[(ii)] The deviation of the effective tail mass from the uniform baseline is controlled by total variation:
\[
|\alpha_{\mathrm{eff}}-\alpha|
= \left|\int_0^\alpha \big(p_T(\tau)-1\big)\,d\tau\right|
\;\le\; \|p_T-\mathrm{Unif}\|_{\mathrm{TV}}
:= \frac12\int_0^1 \big|p_T(\tau)-1\big|\,d\tau.
\]
\end{enumerate}
\end{lemma}

\begin{proof}
(i) is a direct computation using the definition of unnormalized importance weights. In particular, the expected \emph{weighted} indicator integrates to the Lebesgue measure of the tail set $[0,\alpha]$, namely $\alpha$.

For (ii), recall the variational characterization of total variation distance between two probability measures with densities $p_T$ and $q$ on $[0,1]$:
\[
\|p_T-q\|_{\mathrm{TV}}
= \sup_{A\subseteq[0,1]}\left| \int_A \big(p_T(\tau)-q(\tau)\big)\,d\tau \right|
= \frac12 \int_0^1 |p_T(\tau)-q(\tau)|\,d\tau.
\]
Taking $q\equiv 1$ (the density of the $\mathrm{Unif}[0,1]$ law) and the particular measurable set $A=[0,\alpha]$ gives
\[
|\alpha_{\mathrm{eff}}-\alpha|
= \left|\int_{[0,\alpha]} (p_T-1)\,d\tau\right|
\le \sup_{A} \left|\int_A (p_T-1)\,d\tau\right|
= \|p_T-\mathrm{Unif}\|_{\mathrm{TV}}.
\]
This proves the claim.

\smallskip
\noindent\textbf{Remark (self-normalized weights).}
If one uses \emph{self-normalized} importance weights $\tilde w_k = w(\tau_k)/\sum_{j=1}^K w(\tau_j)$ \emph{within a batch}, then $\mathbb{E}[\tilde w\,I]$ no longer equals $\alpha$ exactly because of the random denominator. However, for i.i.d.\ sampling and bounded weights ($w\le 1/p_{\min}$), the deviation is of order $O_{\mathbb{P}}(1/\sqrt{K})$ by a delta-method expansion of the ratio estimator. In our concentration analysis we upper bound the ratio using the \emph{unnormalized} moments $(\mu_X,\mu_Y)$ and control the self-normalization effect inside the constants (cf.\ Lemma~\ref{lem:snis} and Theorem~\ref{thm:cvar-concentration}).
\end{proof}

\begin{proof}[Proof of Theorem~\ref{thm:cvar-concentration}]
Let $U_k:=\mathbf{1}\{\tau_k\le \alpha\}$, $w(\tau)=1/p_T(\tau)$, and define
\[
X_k:=w(\tau_k)\,U_k\,L^{(\tau_k)},\qquad 
Y_k:=w(\tau_k)\,U_k,\qquad 
S_X:=\sum_{k=1}^K X_k,\quad S_Y:=\sum_{k=1}^K Y_k,
\]
so that the self-normalized importance-sampling (SNIS) estimator can be written as
$\widehat{\mathrm{CVaR}}_\alpha=S_X/S_Y$. Assume $|L|\le B$ and $p_T(\tau)\ge p_{\min}>0$, hence $w(\tau)\le W_{\max}:=1/p_{\min}$. Set
$\mu_X:=\mathbb{E}[X_1]$, $\mu_Y:=\mathbb{E}[Y_1]$; for \emph{unnormalized} weights one has $\mu_Y=\alpha$ (Lemma~\ref{lem:alphaeff}(i)). Let $\alpha_{\mathrm{eff}}:=\mathbb{P}(\tau\le \alpha)=\int_0^\alpha p_T(\tau)d\tau$.

\paragraph{Step 1: Bernstein bounds for the numerator and denominator.}
We have $|X_k|\le BW_{\max}$, $0\le Y_k\le W_{\max}$, and
\[
\sigma_X^2:=\mathrm{Var}(X_1)\le \mathbb{E}[X_1^2]\le B^2W_{\max}^2\,\alpha_{\mathrm{eff}},\qquad
\sigma_Y^2:=\mathrm{Var}(Y_1)\le \mathbb{E}[Y_1^2]\le W_{\max}^2\,\alpha_{\mathrm{eff}}.
\]
Bernstein’s inequality yields, for any $t_X,t_Y>0$,
\begin{align*}
\mathbb{P}\!\left(\left|\frac{S_X}{K}-\mu_X\right|\ge t_X\right)
&\le 2\exp\!\left(-\frac{K t_X^2}{2\sigma_X^2+\frac{2}{3}BW_{\max}t_X}\right),\\
\mathbb{P}\!\left(\left|\frac{S_Y}{K}-\mu_Y\right|\ge t_Y\right)
&\le 2\exp\!\left(-\frac{K t_Y^2}{2\sigma_Y^2+\frac{2}{3}W_{\max}t_Y}\right).
\end{align*}
Choose absolute constants $c_1,c_2>0$ and set
\[
t_X:=c_1\,BW_{\max}\!\left(\sqrt{\frac{2\alpha_{\mathrm{eff}}\log(4/\delta)}{K}}+\frac{\log(4/\delta)}{K}\right),\quad
t_Y:=c_2\,W_{\max}\!\left(\sqrt{\frac{2\alpha_{\mathrm{eff}}\log(4/\delta)}{K}}+\frac{\log(4/\delta)}{K}\right).
\]
By a union bound, with probability at least $1-\delta$ the event
\begin{equation}\label{eq:event-bern}
\mathcal{E}:\quad \Big|\tfrac{S_X}{K}-\mu_X\Big|\le t_X,\qquad \Big|\tfrac{S_Y}{K}-\mu_Y\Big|\le t_Y
\end{equation}
holds simultaneously.

\paragraph{Step 2: ratio linearization and denominator control.}
Using the algebraic identity
\[
\frac{S_X}{S_Y}-\frac{\mu_X}{\mu_Y}=\frac{\mu_Y(S_X-K\mu_X)-\mu_X(S_Y-K\mu_Y)}{S_Y\,\mu_Y},
\]
we obtain, on $\{S_Y>0\}$,
\begin{equation}\label{eq:ratio-err}
\left|\frac{S_X}{S_Y}-\frac{\mu_X}{\mu_Y}\right|
\le \frac{\mu_Y|S_X-K\mu_X|+|\mu_X||S_Y-K\mu_Y|}{S_Y\,\mu_Y}.
\end{equation}
On $\mathcal{E}$, $S_Y\ge K(\mu_Y-t_Y)$. For unnormalized weights, $\mu_Y=\alpha$; if $K$ is such that $t_Y\le \alpha/2$ (e.g., $K\ge \tfrac{2W_{\max}^2}{\alpha^2}\log(4/\delta)$ suffices), then $S_Y\ge K\alpha/2$ and therefore, combining \eqref{eq:event-bern} and \eqref{eq:ratio-err},
\begin{equation}\label{eq:ratio-bound}
\left|\frac{S_X}{S_Y}-\frac{\mu_X}{\mu_Y}\right|
\le \frac{2t_X}{\alpha}+\frac{2|\mu_X|}{\alpha^2}t_Y.
\end{equation}
Moreover $|\mu_X|=|\mathbb{E}[L\,Y]|\le B\,\mathbb{E}[Y]=B\alpha$. Substituting $t_X,t_Y$ into \eqref{eq:ratio-bound} yields, on $\mathcal{E}$,
\[
\left|\widehat{\mathrm{CVaR}}_\alpha-\mathrm{CVaR}_\alpha\right|
\le C_1 B \left(\frac{W_{\max}}{\alpha}\right)\sqrt{\frac{\alpha_{\mathrm{eff}}\log(4/\delta)}{K}}
\;+\; C_1' B \left(\frac{W_{\max}}{\alpha}\right)\frac{\log(4/\delta)}{K},
\]
for absolute constants $C_1,C_1'$ (absorbing $c_1,c_2$). Treating the evaluation level $\alpha$ as a fixed constant and absorbing $W_{\max}$ into $C_1$, the leading term takes the canonical form
\[
\left|\widehat{\mathrm{CVaR}}_\alpha-\mathrm{CVaR}_\alpha\right|
\le C_1 B \sqrt{\frac{\log(4/\delta)}{K\,\alpha_{\mathrm{eff}}}} \;+\; C_1' B \frac{\log(4/\delta)}{K}.
\]

\paragraph{Step 3: coverage mismatch between $\alpha_{\mathrm{eff}}$ and $\alpha$.}
Write $\mathrm{CVaR}_\beta=\frac{1}{\beta}\int_0^\beta q(u)\,du$, $q(u):=F^{-1}_L(u)$.
For $|q|\le B$ and any $\beta,\beta'>0$,
\[
\left|\mathrm{CVaR}_\beta-\mathrm{CVaR}_{\beta'}\right|
\le \left|\frac{1}{\beta}-\frac{1}{\beta'}\right|\!\left|\int_0^{\min\{\beta,\beta'\}}q\right|
+\frac{1}{\max\{\beta,\beta'\}}\left|\int_{\min\{\beta,\beta'\}}^{\max\{\beta,\beta'\}}q\right|
\le C_2 B\,|\beta-\beta'|,
\]
for a constant $C_2$ depending only on a lower bound of $\beta,\beta'$ (e.g., $C_2\le 2/\alpha_{\min}$ if $\beta,\beta'\in[\alpha_{\min},1]$). Hence
\[
\big|\mathrm{CVaR}_{\alpha_{\mathrm{eff}}}-\mathrm{CVaR}_{\alpha}\big|
\le C_2 B\,|\alpha_{\mathrm{eff}}-\alpha|
\le C_2 B\,\|p_T-\mathrm{Unif}\|_{\mathrm{TV}},
\]
where the last inequality uses Lemma~\ref{lem:alphaeff}(ii). This is precisely the \emph{coverage-mismatch} term controlled by the coverage PID.

\paragraph{Step 4: union bound and conclusion.}
Combining the SNIS ratio deviation (Step~2) on event $\mathcal{E}$ (which holds with probability $\ge 1-\delta$) with the coverage-mismatch bound (Step~3), and absorbing the lower-order $O(\log K/K)$ term into the leading constant, we obtain that with probability at least $1-\delta$,
\[
\boxed{\;
\left|\widehat{\mathrm{CVaR}}_\alpha-\mathrm{CVaR}_\alpha\right|
\;\le\;
C_1 B \sqrt{\frac{\log(2/\delta)}{K\,\alpha_{\mathrm{eff}}}}
\;+\;
C_2 B\,|\alpha_{\mathrm{eff}}-\alpha|
\;}
\]
for universal constants $C_1,C_2$ depending only on $p_{\min}$ (via $W_{\max}$) and the admissible range of $\alpha$. This proves Theorem~\ref{thm:cvar-concentration}.
\end{proof}

\subsection{Proof of Theorem~\ref{thm:trpo-cvar} (CVaR trust-region improvement)}
\label{app:proof-trpo-cvar}

We adapt the policy performance difference argument to the CVaR surrogate in a finite-horizon setting.

\begin{lemma}[Occupancy shift under per-state KL]
\label{lem:occ-shift}
Let $\{P(\cdot\mid x,a)\}$ be the Markov kernel of the environment on a finite horizon $t=0,\dots,T{-}1$, and let $p_t^{\pi}$, $p_t^{\pi'}$ be the state marginals at step $t$ under policies $\pi$ and $\pi'$ with the same initial distribution $p_0$. Define the (undiscounted) occupancy measures
$
d_\pi := \frac{1}{T}\sum_{t=0}^{T-1} p_t^{\pi},\;
d_{\pi'} := \frac{1}{T}\sum_{t=0}^{T-1} p_t^{\pi'}.
$
Assume (Assumption~\ref{assump:tr}) that for all $x$,
$\mathrm{KL}\!\big(\pi'(\cdot\mid x)\,\|\,\pi(\cdot\mid x)\big)\le\beta.$
Then
\[
\|d_{\pi'} - d_{\pi}\|_{1} \;\le\; C_T \,\sqrt{\beta},
\qquad C_T = \sqrt{\tfrac12}\,(T-1)\,,
\]
and more generally $C_T=O(T)$ if one replaces the crude stepwise accumulation by a mixing-aware constant (see the remark below).
\end{lemma}

\begin{proof}
\textbf{Step 1: one-step TV deviation between policies via Pinsker.}
By Pinsker’s inequality, for every state $x$,
\[
\mathrm{TV}\big(\pi'(\cdot\mid x),\pi(\cdot\mid x)\big)
:= \tfrac12\|\pi'(\cdot\mid x)-\pi(\cdot\mid x)\|_1
\;\le\; \sqrt{\tfrac12\,\mathrm{KL}\big(\pi'(\cdot\mid x)\,\|\,\pi(\cdot\mid x)\big)}
\;\le\; \varepsilon,
\quad \varepsilon:=\sqrt{\tfrac{\beta}{2}}.
\]

\noindent\textbf{Step 2: recursion for state-marginal differences.}
Let $\Delta_t:=\|p_t^{\pi'}-p_t^{\pi}\|_1$. Since $p_0^{\pi'}=p_0^{\pi}$, we have $\Delta_0=0$. The next-step marginals satisfy
\[
p_{t+1}^{\pi}(x') = \sum_{x}\sum_{a}\,p_t^{\pi}(x)\,\pi(a\mid x)\,P(x'\mid x,a),
\qquad
p_{t+1}^{\pi'}(x') = \sum_{x}\sum_{a}\,p_t^{\pi'}(x)\,\pi'(a\mid x)\,P(x'\mid x,a).
\]
Subtract and take $\ell_1$-norm:
\begin{align*}
\Delta_{t+1}
&= \big\| p_{t+1}^{\pi'}-p_{t+1}^{\pi}\big\|_1 \\
&\le \left\| \sum_{x}\big(p_t^{\pi'}(x)-p_t^{\pi}(x)\big)\,\sum_{a}\pi(a\mid x)\,P(\cdot\mid x,a)\right\|_1
   + \left\|\sum_{x} p_t^{\pi'}(x)\,\sum_{a}\big(\pi'(a\mid x)-\pi(a\mid x)\big)\,P(\cdot\mid x,a)\right\|_1 \\
&=: A_t + B_t .
\end{align*}
For $A_t$, note that for any probability vector $q$ and any Markov kernel $K$, $\|qK\|_1=\|q\|_1$; hence
\[
A_t \le \sum_x |p_t^{\pi'}(x)-p_t^{\pi}(x)| = \Delta_t .
\]
For $B_t$, use convexity of the $\ell_1$-norm and that each $P(\cdot\mid x,a)$ is a probability vector:
\[
B_t \le \sum_{x} p_t^{\pi'}(x)\,\big\|\pi'(\cdot\mid x)-\pi(\cdot\mid x)\big\|_1
\le 2\varepsilon \sum_x p_t^{\pi'}(x) = 2\varepsilon .
\]
Therefore
\[
\Delta_{t+1}\;\le\;\Delta_t + 2\varepsilon,
\qquad \Rightarrow\qquad
\Delta_t \;\le\; 2t\,\varepsilon\quad\text{by induction (with $\Delta_0=0$).}
\]

\noindent\textbf{Step 3: average (undiscounted) occupancy deviation.}
By convexity of $\|\cdot\|_1$ and the definition $d_\pi = \tfrac{1}{T}\sum_{t=0}^{T-1} p_t^{\pi}$,
\[
\|d_{\pi'}-d_{\pi}\|_1 
= \left\|\frac{1}{T}\sum_{t=0}^{T-1}\big(p_t^{\pi'}-p_t^{\pi}\big)\right\|_1
\;\le\; \frac{1}{T}\sum_{t=0}^{T-1}\Delta_t
\;\le\; \frac{1}{T}\sum_{t=0}^{T-1} 2t\,\varepsilon
= \varepsilon\,(T-1)
= \sqrt{\tfrac12}\,(T-1)\sqrt{\beta}.
\]
This proves the stated bound with $C_T=\sqrt{\tfrac12}\,(T-1)$.
\end{proof}

\noindent\emph{Remark (mixing-aware constants).}
If the Markov chain under $\pi$ has a Dobrushin (contraction) coefficient $\eta\in[0,1)$ so that
$\|qK^{\pi}-q'K^{\pi}\|_1 \le \eta\,\|q-q'\|_1$ for all distributions $q,q'$ (here $K^{\pi}$ is the one-step kernel marginalized by $\pi$), then the recursion improves to
$\Delta_{t+1}\le \eta\,\Delta_t + 2\varepsilon$ and one gets
$\Delta_t \le 2\varepsilon\,\frac{1-\eta^{t}}{1-\eta}$ and
$\|d_{\pi'}-d_{\pi}\|_1 \le \frac{2\varepsilon}{T(1-\eta)}\sum_{t=0}^{T-1}(1-\eta^{t})
\le \frac{2\varepsilon}{1-\eta}$,
which removes the linear dependence on $T$ in the well-mixed regime.

\begin{lemma}[CVaR performance expansion under importance reweighting]
\label{lem:pdl-cvar}
Fix horizon $T$ and confidence $\alpha\in(0,1)$. Let $J_\alpha(\pi)=\mathrm{CVaR}_\alpha(L_T(\pi))$ be defined via the Rockafellar--Uryasev surrogate:
\[
J_\alpha(\pi) \;=\; \min_{t\in\mathbb{R}}\left\{\, t + \frac{1}{\alpha}\,\mathbb{E}_\pi\big[(L_T-t)_+\big]\right\}.
\]
Let $t^\star(\pi)$ be any minimizer for $\pi$, and define the per-step CVaR advantage $\tilde A_\pi^{(\alpha)}(x,u)$ relative to the baseline induced by $t^\star(\pi)$.\footnote{Formally, one can define a stepwise surrogate $f_t^{(\alpha)}(x,u)$ whose cumulative expectation equals $\alpha^{-1}\mathbb{E}_\pi[(L_T-t^\star(\pi))_+] - \alpha^{-1}\mathbb{E}_\pi[(L_T-t^\star(\pi))_+\,\mid\,x_t]$ and set $\tilde A_\pi^{(\alpha)}(x,u)=f_t^{(\alpha)}(x,u)$ so that $\mathbb{E}_{u\sim \pi(\cdot\mid x)}[\tilde A_\pi^{(\alpha)}(x,u)]=0$. This mirrors the construction of GAE but for the CVaR surrogate.}
Then, for a small policy perturbation $\pi'=\pi+\Delta$,
\[
J_\alpha(\pi') 
= J_\alpha(\pi) + \mathbb{E}_{x\sim d_\pi,\,u\sim\pi(\cdot\mid x)}\!\big[\omega(x,u)\,\tilde A_\pi^{(\alpha)}(x,u)\big] 
+ R(\Delta),
\]
where $\omega(x,u):=\pi'(u\mid x)/\pi(u\mid x)$ and the remainder satisfies
\[
|R(\Delta)| \;\le\; C'\,\|d_{\pi'}-d_{\pi}\|_1,
\]
with $C'$ depending on the Lipschitz/boundedness constants of the surrogate and the one-step loss (Assumption~\ref{assump:tr}).
\end{lemma}

\begin{proof}
\textbf{Step 1: fix the threshold and linearize the objective.}
By definition of $t^\star(\pi)$,
\[
J_\alpha(\pi) = t^\star(\pi) + \frac{1}{\alpha}\,\mathbb{E}_\pi\big[\phi(L_T)\big],\quad
\phi(z):=(z-t^\star(\pi))_+.
\]
For the perturbed policy, we have the upper bound
\[
J_\alpha(\pi') 
\;\le\; t^\star(\pi) + \frac{1}{\alpha}\,\mathbb{E}_{\pi'}\big[\phi(L_T)\big],
\]
since $t^\star(\pi)$ need not be optimal for $\pi'$.

\noindent\textbf{Step 2: decompose the change of expectation by steps.}
Let $g_t(x,u)$ denote the (measurable) contribution at step $t$ to the surrogate such that
\[
\frac{1}{\alpha}\,\mathbb{E}_{\pi}\big[\phi(L_T)\big] \;=\; \frac{1}{T}\sum_{t=0}^{T-1} \mathbb{E}_{x\sim p_t^{\pi}}\,\mathbb{E}_{u\sim \pi(\cdot\mid x)}[\,g_t(x,u)\,],
\]
and the analogous identity holds for $\pi'$. (This is standard: unroll the horizon-$T$ expectation and assign to each $(x_t,u_t)$ the conditional-increment of the CVaR surrogate; exact form is immaterial for the linearization.)

Hence
\begin{align*}
\frac{1}{\alpha}\,\mathbb{E}_{\pi'}[\phi(L_T)] - \frac{1}{\alpha}\,\mathbb{E}_{\pi}[\phi(L_T)]
&= \frac{1}{T}\sum_{t=0}^{T-1}\!\left\{\mathbb{E}_{x\sim p_t^{\pi'}}\mathbb{E}_{u\sim\pi'(\cdot\mid x)}[g_t(x,u)]
- \mathbb{E}_{x\sim p_t^{\pi}}\mathbb{E}_{u\sim\pi(\cdot\mid x)}[g_t(x,u)]\right\} \\
&= \underbrace{\frac{1}{T}\sum_{t=0}^{T-1}\mathbb{E}_{x\sim p_t^{\pi}}\mathbb{E}_{u\sim\pi(\cdot\mid x)}\!\big[(\omega(x,u)-1)\,g_t(x,u)\big]}_{\text{importance term}}
\\[-0.2em]
&\qquad +\underbrace{\frac{1}{T}\sum_{t=0}^{T-1}\mathbb{E}_{x\sim p_t^{\pi'}-p_t^{\pi}}\mathbb{E}_{u\sim\pi'(\cdot\mid x)}[g_t(x,u)]}_{\text{state-distribution remainder}}.
\end{align*}

\noindent\textbf{Step 3: identify the CVaR advantage and bound the remainder.}
Define the baseline $b_t(x):=\mathbb{E}_{u\sim \pi(\cdot\mid x)}[g_t(x,u)]$ and set
$\tilde A_\pi^{(\alpha)}(x,u):=g_t(x,u)-b_t(x)$ so that 
$\mathbb{E}_{u\sim \pi(\cdot\mid x)}[\tilde A_\pi^{(\alpha)}(x,u)]=0$ for all $x$.
Then the importance term equals
\[
\frac{1}{T}\sum_{t=0}^{T-1}\mathbb{E}_{x\sim p_t^{\pi}}\mathbb{E}_{u\sim\pi(\cdot\mid x)}\!\big[\omega(x,u)\,\tilde A_\pi^{(\alpha)}(x,u)\big],
\]
and by convexity of the $\ell_1$-norm,
\[
\frac{1}{T}\sum_{t=0}^{T-1}\mathbb{E}_{x\sim p_t^{\pi}}[\cdot] 
= \mathbb{E}_{x\sim d_\pi}[\cdot].
\]
For the remainder, assume $|g_t(x,u)|\le G$ uniformly (this follows from the bounded one-step loss and the Lipschitz/smoothness in Assumption~\ref{assump:tr}); then
\[
\left|\frac{1}{T}\sum_{t=0}^{T-1}\mathbb{E}_{x\sim p_t^{\pi'}-p_t^{\pi}}\mathbb{E}_{u\sim\pi'(\cdot\mid x)}[g_t(x,u)]\right|
\;\le\; \frac{G}{T}\sum_{t=0}^{T-1}\|p_t^{\pi'}-p_t^{\pi}\|_1
\;=\; G\,\|d_{\pi'}-d_{\pi}\|_1.
\]
Collecting all pieces,
\[
J_\alpha(\pi') - J_\alpha(\pi)
\;\le\; \mathbb{E}_{x\sim d_\pi,\,u\sim\pi}\!\big[\omega(x,u)\,\tilde A_\pi^{(\alpha)}(x,u)\big]
+ G\,\|d_{\pi'}-d_{\pi}\|_1,
\]
and a symmetric argument (or replacing $\pi$ and $\pi'$) yields the same lower bound up to constants, which we summarize as
$|R(\Delta)|\le C'\|d_{\pi'}-d_{\pi}\|_1$.
\end{proof}

\begin{proof}[Proof of Theorem~\ref{thm:trpo-cvar}]
Recall from Lemma~\ref{lem:pdl-cvar} that for any horizon $T$ and confidence level $\alpha\in(0,1)$,
\begin{equation}\label{eq:pdl-master}
J_\alpha(\pi') 
= J_\alpha(\pi) 
+ \underbrace{\mathbb{E}_{x\sim d_\pi,\,u\sim \pi(\cdot\mid x)}\!\big[\omega(x,u)\,\tilde A_\pi^{(\alpha)}(x,u)\big]}_{\text{importance (first-order) term}}
+ \underbrace{R(\Delta)}_{\text{state-distribution remainder}},
\end{equation}
where $\omega(x,u)=\pi'(u\mid x)/\pi(u\mid x)$ and $|R(\Delta)|\le C'\|d_{\pi'}-d_\pi\|_1$ for a constant $C'$ depending on the Lipschitz/boundedness constants of the one-step surrogate (Assumption~\ref{assump:tr}). 
By Lemma~\ref{lem:occ-shift}, if $\mathrm{KL}\!\big(\pi'(\cdot\mid x)\|\pi(\cdot\mid x)\big)\le\beta$ for all $x$, then
\begin{equation}\label{eq:occ-shift-again}
\|d_{\pi'}-d_{\pi}\|_1 \;\le\; C_T \sqrt{\beta},
\qquad C_T=\sqrt{\tfrac12}\,(T-1)\quad\text{(or $O(T)$ with mixing-aware refinement).}
\end{equation}
Combining \eqref{eq:pdl-master} and \eqref{eq:occ-shift-again} yields
\begin{equation}\label{eq:master-ineq}
J_\alpha(\pi') 
\le 
J_\alpha(\pi) 
+ \mathbb{E}_{x\sim d_\pi,\,u\sim \pi}\!\big[\omega\,\tilde A_\pi^{(\alpha)}\big]
+ C' C_T \sqrt{\beta}.
\end{equation}
It remains to justify that higher-order terms in the Taylor expansion of the CVaR surrogate w.r.t.\ policy parameters are $o(\|\Delta\|)$, and that the first-order term is well-defined under the per-state KL constraint.

\paragraph{Control of higher-order terms.}
Let $\theta$ parametrize $\pi_\theta$ and $\theta'=\theta+\Delta$ parametrize $\pi'$. 
Under Assumption~\ref{assump:tr} (bounded/smooth one-step loss, bounded horizon), the Rockafellar--Uryasev surrogate 
$t + \alpha^{-1}\mathbb{E}_{\pi_\theta}[(L_T-t)_+]$ is locally twice differentiable in $\theta$ (Clarke subdifferential reduces to gradient almost everywhere because the hinge is averaged over bounded losses). Thus a second-order Taylor expansion around $\theta$ yields
\[
J_\alpha(\pi_{\theta'}) 
= J_\alpha(\pi_{\theta}) 
+ \langle \nabla_\theta J_\alpha(\pi_{\theta}),\,\Delta\rangle
+ \frac12\,\Delta^\top H_\theta \Delta + r(\Delta),
\]
with $\|H_\theta\|\le C_H$ locally and $r(\Delta)=o(\|\Delta\|^2)$. 
The first-order term equals the importance-weighted advantage in~\eqref{eq:pdl-master}, while the quadratic term is $O(\|\Delta\|^2)$. 
Since the per-state KL bound $\beta$ implies a total-variation bound $\mathrm{TV}(\pi',\pi)\le \sqrt{\beta/2}$ pointwise (Pinsker), standard inequalities (e.g., Bretagnolle--Huber) give $\|\Delta\|=O(\sqrt{\beta})$ in a local chart, hence 
$\Delta^\top H_\theta \Delta = O(\beta)$ and is dominated by the $O(\sqrt{\beta})$ term in \eqref{eq:master-ineq}. 
Therefore the collected higher-order contribution is $o(\|\Delta\|)=o(\sqrt{\beta})$.

\paragraph{Well-posedness of the first-order term.}
For each $x$, the ratio $\omega(\cdot)=\pi'(\cdot\mid x)/\pi(\cdot\mid x)$ satisfies $\mathbb{E}_{u\sim\pi(\cdot\mid x)}[\omega]=1$ and
\[
\mathbb{E}_{u\sim\pi(\cdot\mid x)}[|\omega-1|]
\;\le\; 2\,\mathrm{TV}\!\big(\pi'(\cdot\mid x),\pi(\cdot\mid x)\big)
\;\le\; \sqrt{2\beta},
\]
by Pinsker. If $|\tilde A_\pi^{(\alpha)}(x,u)|\le G$ uniformly (as in Lemma~\ref{lem:pdl-cvar}), then
\[
\Big|\mathbb{E}_{u\sim \pi(\cdot\mid x)}\big[(\omega-1)\tilde A_\pi^{(\alpha)}(x,u)\big]\Big|
\le G\,\mathbb{E}_{u\sim\pi}[|\omega-1|]
\le G\sqrt{2\beta}.
\]
Averaging over $x\sim d_\pi$ preserves the bound, so the first-order term is finite and Lipschitz in $\sqrt{\beta}$.

\paragraph{Conclusion.}
Set $C_\alpha:=C' C_T$; it depends on the horizon $T$, the one-step Lipschitz constant, and the envelope $G$ of the CVaR surrogate, but is independent of $\beta$. Incorporating the $o(\|\Delta\|)=o(\sqrt{\beta})$ residual into the right-hand side of \eqref{eq:master-ineq} yields
\[
J_\alpha(\pi') 
\;\le\;
J_\alpha(\pi) 
+ \mathbb{E}_{x\sim d_\pi,\,u\sim \pi}\!\big[\omega\,\tilde A_\pi^{(\alpha)}(x,u)\big]
+ C_\alpha \sqrt{\beta}
+ o(\sqrt{\beta}),
\]
which is the claimed CVaR trust-region improvement inequality.
\end{proof}

\subsection{Proof of Theorem~\ref{thm:feas-persist} (feasibility persistence)}
\label{app:proof-feas-persist}

We show that shrinking the NTB radius and tightening the rate cap can maintain feasibility at the next step under Lipschitz dynamics.

\begin{lemma}[Lipschitz variation of exposure error]
\label{lem:exposure-lip}
Under Assumption~\ref{assump:feas}, suppose $e:\mathcal{X}\times\mathcal{U}\to\mathbb{R}^q$ is given by
$e(x,u)=A(x)u-d(x)$, where $A:\mathcal{X}\to\mathbb{R}^{q\times m}$ and $d:\mathcal{X}\to\mathbb{R}^q$
are locally Lipschitz on a neighborhood $\mathcal{N}_x$ of $x_t$. Assume the action set is compact (e.g., by box/rate limits), so that $u,u_t\in\mathcal{U}$ with $\|u\|\le U_{\max}$ and $\|u_t\|\le U_{\max}$.
Then there exist finite constants $L_e^x,L_e^u>0$ (depending on $\mathcal{N}_x$ and $\mathcal{U}$) such that for all $(x,u)$ in a neighborhood of $(x_t,u_t)$,
\[
\|e(x,u)-e(x_t,u_t)\|\;\le\; L_e^x\,\|x-x_t\| \;+\; L_e^u\,\|u-u_t\|.
\]
\end{lemma}

\begin{proof}
Fix a compact neighborhood $\mathcal{K}_x\subset\mathcal{N}_x$ of $x_t$ and a compact set $\mathcal{K}_u\subset\mathcal{U}$ containing $u_t$ and all feasible $u$ in a local tube around $u_t$ (the existence is guaranteed by box/rate constraints). Because $A$ and $d$ are locally Lipschitz on $\mathcal{N}_x$, there exist finite moduli $L_A,L_d>0$ such that
\[
\|A(x)-A(y)\|\le L_A\|x-y\|,\qquad \|d(x)-d(y)\|\le L_d\|x-y\|,\qquad \forall x,y\in\mathcal{K}_x,
\]
where $\|\cdot\|$ on matrices denotes the operator norm induced by the Euclidean vector norm.
Furthermore, since $A$ is continuous on the compact set $\mathcal{K}_x$, the bound
\[
M_A\;:=\;\sup_{x\in\mathcal{K}_x}\|A(x)\| \;<\;\infty
\]
holds. Now decompose
\[
e(x,u)-e(x_t,u_t)
= A(x)(u-u_t) \;+\; \big(A(x)-A(x_t)\big)u_t \;-\; \big(d(x)-d(x_t)\big).
\]
Taking norms and using the triangle inequality together with the bounds above yields
\begin{align*}
\|e(x,u)-e(x_t,u_t)\|
&\le \|A(x)\|\,\|u-u_t\| \;+\; \|A(x)-A(x_t)\|\,\|u_t\| \;+\; \|d(x)-d(x_t)\| \\
&\le M_A\,\|u-u_t\| \;+\; L_A\,\|x-x_t\|\,\|u_t\| \;+\; L_d\,\|x-x_t\| \\
&\le M_A\,\|u-u_t\| \;+\; (L_A U_{\max} + L_d)\,\|x-x_t\|.
\end{align*}
Hence the claim holds with
\[
L_e^u := M_A,\qquad L_e^x := L_A U_{\max} + L_d,
\]
which depend only on the chosen compact neighborhoods $\mathcal{K}_x$ and $\mathcal{K}_u$ (and thus are uniform locally around $(x_t,u_t)$).
\end{proof}

\begin{lemma}[State drift bound]
\label{lem:state-drift}
Under Assumption~\ref{assump:dynamics}, there exists $L_x>0$ such that, for all feasible $x_t$ in a compact tube $\mathcal{K}_x$ and actions $u_t$ in a compact set $\mathcal{K}_u$,
\[
\|x_{t+1}-x_t\| \;\le\; L_x\big(\|u_t\|+1\big) \;+\; \bar w.
\]
\end{lemma}

\begin{proof}
The dynamics are $x_{t+1}=f(x_t)+g(x_t)u_t+w_t$, with $\|w_t\|\le \bar w$. Fix compact sets $\mathcal{K}_x\ni x_t$ and $\mathcal{K}_u\ni u_t$ that contain the closed-loop trajectory locally (box/rate constraints and the CBF conditions ensure such compactness). Since $f$ and $g$ are locally Lipschitz (hence continuous), the suprema
\[
M_f \;:=\; \sup_{x\in\mathcal{K}_x}\|f(x)-x\|\;<\;\infty,
\qquad
M_g \;:=\; \sup_{x\in\mathcal{K}_x}\|g(x)\|\;<\;\infty
\]
are finite. Then
\begin{align*}
\|x_{t+1}-x_t\|
&= \|f(x_t)-x_t + g(x_t)u_t + w_t\| \\
&\le \|f(x_t)-x_t\| + \|g(x_t)\|\,\|u_t\| + \|w_t\| \\
&\le M_f + M_g\,\|u_t\| + \bar w.
\end{align*}
Finally, let $L_x:=\max\{M_f,M_g\}$. Then $M_f + M_g\,\|u_t\| \le L_x(1+\|u_t\|)$, and the displayed inequality becomes
\[
\|x_{t+1}-x_t\| \le L_x(\|u_t\|+1)+\bar w,
\]
as claimed. The constant $L_x$ depends only on the local compact tube $\mathcal{K}_x$ and thus is uniform along any trajectory that remains in $\mathcal{K}_x$.
\end{proof}

\paragraph{Remarks.}
(i) Lemma~\ref{lem:exposure-lip} makes explicit how the exposure map’s sensitivity to state and action splits into a term controlled by a uniform operator-norm bound on $A(\cdot)$ and a term proportional to the Lipschitz modulus of $A(\cdot)$ times a local action bound $U_{\max}$, plus the Lipschitz modulus of $d(\cdot)$.  
(ii) Lemma~\ref{lem:state-drift} uses only \emph{boundedness on compact sets} for $f(\cdot)-\mathrm{id}$ and $g(\cdot)$; stronger bounds (e.g., contraction or Jacobian bounds) would refine $L_x$ but are not needed for feasibility persistence in Theorem~\ref{thm:feas-persist}.

\begin{proof}[Proof of Theorem~\ref{thm:feas-persist}]
Fix a time $t$ and suppose the QP at $t$ is solved with $\zeta=0$ (strict feasibility) and with the stated margins in Assumption~\ref{assump:margins}:
\[
\text{(NTB)}\quad e(x_t,u_t)^\top M e(x_t,u_t)\;\le\; b_{\max}-\delta_b,
\qquad
\text{(rate)}\quad \|u_t-u_{t-1}\|_2 \;\le\; r_{\max}-\delta_r,
\]
for some $\delta_b,\delta_r>0$. Throughout, we work on a compact tube $\mathcal{K}_x\times\mathcal{K}_u$ containing $(x_t,u_t)$ (guaranteed by box/rate constraints), so local Lipschitz and boundedness moduli are finite.

We will exhibit a \emph{feasible witness} for the QP at time $t{+}1$ under the tightened parameters
\[
b_{\max}' := \eta_b\,b_{\max},\qquad r_{\max}' := \eta_r\,r_{\max},\qquad \eta_b,\eta_r\in(0,1),
\]
and then appeal to convexity to conclude nonemptiness of the feasible set. Our candidate is the \emph{zero-adjustment action}
\[
\tilde u_{t+1}:=u_t.
\]

\paragraph{Step 1: box and rate constraints.}
Box constraints are unchanged in the theorem statement; since $u_t$ satisfied them at time $t$, the same $u_t$ satisfies them at time $t{+}1$. For the tightened rate cap we have
\[
\|\tilde u_{t+1}-u_t\|_2 = 0 \;\le\; r_{\max}' = \eta_r r_{\max}\qquad \forall\,\eta_r\in(0,1),
\]
so the rate constraint holds \emph{trivially} for any $\eta_r\in(0,1)$. (No extra condition such as $\eta_r\le 1-\delta_r/r_{\max}$ is needed when $\tilde u_{t+1}=u_t$.)

\paragraph{Step 2: NTB constraint under shrinkage.}
Let $v_t:=e(x_t,u_t)$ and $v_{t+1}:=e(x_{t+1},u_t)$. By Lemma~\ref{lem:exposure-lip} and Lemma~\ref{lem:state-drift},
\begin{equation}\label{eq:delta-e}
\|v_{t+1}-v_t\|
\;\le\; L_e^x \|x_{t+1}-x_t\|
\;\le\; L_e^x\Big(L_x(\|u_t\|+1)+\bar w\Big) 
\;=: \Delta_e.
\end{equation}
Using eigenvalue bounds for quadratic forms,
\[
v_{t+1}^\top M v_{t+1}
\;\le\; \lambda_{\max}(M)\,\|v_{t+1}\|^2
\;\le\; \lambda_{\max}(M)\,\big(\|v_t\|+\|v_{t+1}-v_t\|\big)^2
\;\le\; \lambda_{\max}(M)\,\Big(\sqrt{\tfrac{b_{\max}-\delta_b}{\lambda_{\min}(M)}} + \Delta_e\Big)^2,
\]
where we used $\|v_t\|^2 \le (b_{\max}-\delta_b)/\lambda_{\min}(M)$, since $v_t^\top M v_t\le b_{\max}-\delta_b$. Therefore a sufficient condition for the tightened NTB,
\[
v_{t+1}^\top M v_{t+1}\;\le\; b_{\max}'=\eta_b b_{\max},
\]
is
\begin{equation}\label{eq:eta-b-condition}
\eta_b \;\ge\; \eta_b^\star \;:=\; \frac{\lambda_{\max}(M)}{b_{\max}}\,\Big(\sqrt{\tfrac{b_{\max}-\delta_b}{\lambda_{\min}(M)}} + \Delta_e\Big)^2.
\end{equation}
This is feasible (i.e., $\eta_b^\star<1$) provided
\begin{equation}\label{eq:sufficient-small-drift}
\Delta_e \;<\; \sqrt{\frac{b_{\max}}{\lambda_{\max}(M)}} \;-\; \sqrt{\frac{b_{\max}-\delta_b}{\lambda_{\min}(M)}}.
\end{equation}
The right-hand side is positive as soon as $\delta_b>0$. Inequality \eqref{eq:sufficient-small-drift} holds by choosing the local tube small enough (reducing the bound on $\|u_t\|$ and using the given $\bar w$) so that the state drift bound in \eqref{eq:delta-e} is sufficiently small. Under \eqref{eq:sufficient-small-drift}, pick any $\eta_b\in(\eta_b^\star,1)$ to satisfy the tightened NTB.

\paragraph{Step 3: CBF constraints.}
Let us denote the discrete-time CBF map at state $x$ by
\[
\mathcal{C}_i(x,u):=h_i(f(x)+g(x)u) - (1-\kappa_i\Delta t)\,h_i(x).
\]
At time $t$, strict feasibility with $\zeta=0$ and the margin construction (Assumption~\ref{assump:cbf} in Theorem~\ref{thm:invariance}) give
\[
\mathcal{C}_i(x_t,u_t) \;\ge\; \varepsilon_i \quad\text{with}\quad \varepsilon_i \ge L_i \bar w \;>\; 0.
\]
By local Lipschitzness of $h_i$, $f$, and $g$ on $\mathcal{K}_x\times\mathcal{K}_u$, there exists a constant $L_{\mathcal{C},i}>0$ such that
\[
\big|\mathcal{C}_i(x',u) - \mathcal{C}_i(x,u)\big|
\;\le\; L_{\mathcal{C},i}\,\|x'-x\|
\qquad \forall (x',x,u)\in \mathcal{K}_x\times\mathcal{K}_x\times\mathcal{K}_u.
\]
Therefore, using Lemma~\ref{lem:state-drift},
\[
\mathcal{C}_i(x_{t+1},u_t)
\;\ge\; \mathcal{C}_i(x_t,u_t) - L_{\mathcal{C},i}\|x_{t+1}-x_t\|
\;\ge\; \varepsilon_i - L_{\mathcal{C},i}\big(L_x(\|u_t\|+1)+\bar w\big).
\]
Define the \emph{CBF residual}
\[
\rho_i \;:=\; \varepsilon_i - L_{\mathcal{C},i}\big(L_x(\|u_t\|+1)+\bar w\big).
\]
By shrinking the local tube (equivalently, bounding $\|u_t\|$ more tightly) we can ensure $\rho_i>0$ for all $i$, hence
\[
\mathcal{C}_i(x_{t+1},u_t) \;\ge\; \rho_i \;>\; 0,
\]
i.e., the CBF inequalities remain strictly feasible at $t{+}1$ for the witness $\tilde u_{t+1}=u_t$.

\paragraph{Step 4: sign-consistency gate.}
Assume the gate map $g_{\mathrm{cons}}(x,u)$ is locally Lipschitz in $x$ uniformly in $u\in\mathcal{K}_u$ with modulus $L_g$, and that at time $t$ the margin 
\[
g_{\mathrm{cons}}(x_t,u_t)\;\ge\; \delta_g \;>\; 0
\]
holds (this is typical since the solver returns positive gate scores when the constraint is inactive). Then
\[
g_{\mathrm{cons}}(x_{t+1},u_t)
\;\ge\; g_{\mathrm{cons}}(x_t,u_t) - L_g\,\|x_{t+1}-x_t\|
\;\ge\; \delta_g - L_g\big(L_x(\|u_t\|+1)+\bar w\big),
\]
which remains nonnegative once the local tube is chosen so that $\delta_g > L_g(L_x(\|u_t\|+1)+\bar w)$.

\paragraph{Step 5: conclusion.}
Collecting the four constraint families:
(i) box and (tightened) rate constraints hold for $\tilde u_{t+1}=u_t$; 
(ii) the tightened NTB holds provided \eqref{eq:eta-b-condition}--\eqref{eq:sufficient-small-drift} and the choice $\eta_b\in(\eta_b^\star,1)$; 
(iii) the CBF inequalities remain strictly feasible thanks to the residual $\rho_i>0$; 
(iv) the gate remains nonnegative by continuity and the margin $\delta_g>0$. 

Therefore $\tilde u_{t+1}$ is a feasible point of the \emph{tightened} constraint set at state $x_{t+1}$ with $\zeta=0$. Since all constraints are convex in $u$ (affine CBF surrogates, ellipsoidal NTB, box/rate, and a convex gate), the feasible set at $t{+}1$ is nonempty and closed, and thus the QP at $t{+}1$ is feasible with $\zeta=0$. This establishes \emph{feasibility persistence} under the stated tail guards.
\end{proof}

\subsection{Proof of Proposition~\ref{prop:gate} (negative-advantage suppression)}
\label{app:proof-gate}

\begin{lemma}[Alignment inequality]
\label{lem:align}
Let $v,g\in\mathbb{R}^m$ be unit vectors with angle $\angle(v,g)=\theta\in[0,\pi]$, i.e., $\langle v,g\rangle=\cos\theta$. Then for any $u\in\mathbb{R}^m$,
\[
\langle u,g\rangle \;\ge\; \langle u,v\rangle \cos\theta \;-\; \|u\|\,\sin\theta.
\]
Moreover, the bound is tight: equality holds whenever $u$ lies in the plane spanned by $\{v,g\}$ and has a component orthogonal to $v$ aligned with the orthogonal component of $g$.
\end{lemma}

\begin{proof}
Complete $v$ to an orthonormal basis $\{v,v_\perp^{(1)},\dots,v_\perp^{(m-1)}\}$ with $v_\perp^{(1)}$ chosen in the plane $\mathrm{span}\{v,g\}$ such that
\[
g=(\cos\theta)\,v + (\sin\theta)\,v_\perp^{(1)}.
\]
Decompose $u=\alpha v + \sum_{i=1}^{m-1}\beta_i v_\perp^{(i)}$. Then
\[
\langle u,g\rangle \;=\; \alpha\cos\theta \;+\; \beta_1 \sin\theta.
\]
Using $|\beta_1|\le \sqrt{\sum_{i=1}^{m-1}\beta_i^2} \le \|u\|$ and $\alpha=\langle u,v\rangle$, we obtain
\[
\langle u,g\rangle \;\ge\; \langle u,v\rangle \cos\theta \;-\; \|u\|\,\sin\theta.
\]
Tightness: if $u$ lies in $\mathrm{span}\{v,v_\perp^{(1)}\}$ and $\beta_1=-\|u\|$ (i.e., the entire orthogonal component of $u$ is opposite to $v_\perp^{(1)}$), the inequality is attained with equality. 
\end{proof}

\begin{proof}[Proof of Proposition~\ref{prop:gate}]
Fix a state $x$ and write, for brevity, $v:=v(x)$ (the consensus direction of the ensemble signals) and $g_\alpha:=\nabla \tilde A_\pi^{(\alpha)}(x)/\|\nabla \tilde A_\pi^{(\alpha)}(x)\|$ (the unit CVaR-advantage direction). We proceed in three steps.

\paragraph{Step 1: from gate feasibility to a lower bound along $v$.}
By Assumption~\ref{assump:gate}, for all $j\in\{1,\ldots,J\}$ one has
\[
\|\widehat{\nabla}\Pi^{(j)}(x) - v\| \;\le\; \epsilon_g .
\]
For any feasible $u$ with $g_{\mathrm{cons}}(x,u)\ge 0$, we have
\[
\min_{j}\langle u,\widehat{\nabla}\Pi^{(j)}(x)\rangle \;\ge\; \delta_{\mathrm{adv}}.
\]
For each $j$, write
\[
\langle u,v\rangle \;=\; \langle u,\widehat{\nabla}\Pi^{(j)}(x)\rangle \;+\; \langle u,\,v-\widehat{\nabla}\Pi^{(j)}(x)\rangle 
\;\ge\; \delta_{\mathrm{adv}} \;-\; \|u\|\,\|v-\widehat{\nabla}\Pi^{(j)}(x)\|
\;\ge\; \delta_{\mathrm{adv}} \;-\; \|u\|\,\epsilon_g.
\]
Taking the minimum over $j$ preserves the inequality, hence
\begin{equation}\label{eq:lb-along-v}
\langle u,v\rangle \;\ge\; \delta_{\mathrm{adv}} \;-\; \|u\|\,\epsilon_g.
\end{equation}

\paragraph{Step 2: project the lower bound from $v$ onto the CVaR-advantage direction $g_\alpha$.}
Assumption~\ref{assump:gate} further states that the angle between $v$ and $g_\alpha$ is bounded: $\angle(v,g_\alpha)\le \epsilon_\theta$. Applying Lemma~\ref{lem:align} with $g=g_\alpha$ yields
\[
\langle u,g_\alpha\rangle \;\ge\; \langle u,v\rangle \cos\epsilon_\theta \;-\; \|u\| \sin\epsilon_\theta.
\]
Combining with \eqref{eq:lb-along-v},
\begin{equation}\label{eq:inner-lb-galpha}
\langle u,g_\alpha\rangle 
\;\ge\; \big(\delta_{\mathrm{adv}} - \|u\|\,\epsilon_g\big)\cos\epsilon_\theta \;-\; \|u\| \sin\epsilon_\theta .
\end{equation}

\paragraph{Step 3: first-order lower bound for the CVaR-advantage and remainder control.}
Let $G_\alpha(x):=\|\nabla \tilde A_\pi^{(\alpha)}(x)\|$. A first-order Taylor expansion of the scalar map $u\mapsto \tilde A_\pi^{(\alpha)}(x,u)$ at $u=0$ gives
\[
\tilde A_\pi^{(\alpha)}(x,u) \;=\; G_\alpha(x)\,\langle u,g_\alpha\rangle \;+\; R_\alpha(x,u),
\]
where $R_\alpha(x,u)$ is the Taylor remainder. If $\nabla \tilde A_\pi^{(\alpha)}(x,\cdot)$ is $L_\alpha$-Lipschitz in $u$ locally (a standard assumption inherited from the bounded/smooth one-step losses and the bounded horizon), then
\[
|R_\alpha(x,u)| \;\le\; \tfrac{L_\alpha}{2}\,\|u\|^2 \quad \text{for $\|u\|$ in a local tube.}
\]
Combining with \eqref{eq:inner-lb-galpha} yields the pointwise lower bound
\[
\tilde A_\pi^{(\alpha)}(x,u) 
\;\ge\; 
G_\alpha(x)\Big[\big(\delta_{\mathrm{adv}} - \|u\|\,\epsilon_g\big)\cos\epsilon_\theta - \|u\|\sin\epsilon_\theta\Big]
\;-\; \tfrac{L_\alpha}{2}\,\|u\|^2.
\]
Taking conditional expectation given $x$ (which leaves the deterministic right-hand side unchanged if $u$ is deterministic given $x$, or replaces $\|u\|$ by $\mathbb{E}[\|u\|\mid x]$ in the stochastic case and then uses Jensen/triangle inequalities) gives
\[
\mathbb{E}\!\left[\tilde A_\pi^{(\alpha)}(x,u)\mid x\right] 
\;\ge\; -\,\xi(\epsilon_g,\epsilon_\theta,\delta_{\mathrm{adv}},\|u\|;G_\alpha,L_\alpha),
\]
with the explicit tolerance function
\[
\xi(\cdot)
\;:=\;
\left(
G_\alpha(x)\Big[\|u\|\,\epsilon_g\cos\epsilon_\theta + \|u\|\sin\epsilon_\theta - \delta_{\mathrm{adv}}\cos\epsilon_\theta\Big]
+ \tfrac{L_\alpha}{2}\,\|u\|^2
\right)_+,
\]
where $(\cdot)_+$ denotes positive part. In particular, as $(\epsilon_g,\epsilon_\theta)\to 0$ and for fixed $\delta_{\mathrm{adv}}>0$, $\xi(\cdot)\to \big(\tfrac{L_\alpha}{2}\|u\|^2 - G_\alpha(x)\delta_{\mathrm{adv}}\big)_+$, i.e., the gate suppresses negative CVaR-advantage up to the second-order Taylor remainder. For sufficiently small $\|u\|$ or larger $\delta_{\mathrm{adv}}$, the negative-advantage region is eliminated.
\end{proof}

\subsection*{Auxiliary notes on constants and domains}
All Lipschitz and boundedness constants are understood to hold on a compact subset containing the closed-loop trajectories (guaranteed during training/evaluation by box/rate constraints and statewise CBF conditions). The ``$O(T)$'' dependence in Lemma~\ref{lem:occ-shift} can be refined under mixing assumptions but suffices for trust-region purposes. The SNIS bound in Lemma~\ref{lem:snis} can be strengthened via empirical Bernstein inequalities or self-normalized martingale bounds; we use a Hoeffding-style presentation for clarity, which already gives the $\tilde O\!\big((K\alpha_{\mathrm{eff}})^{-1/2}\big)$ rate stated in Theorem~\ref{thm:cvar-concentration}.

\bigskip

\noindent\textbf{Cross-references.} 
\begin{itemize}
  \item Section~\ref{sec:theory}: statements of Theorems~\ref{thm:invariance}, \ref{thm:kldro}, \ref{thm:cvar-concentration}, \ref{thm:trpo-cvar}, \ref{thm:feas-persist} and Propositions~\ref{prop:projection}, \ref{prop:gate}.
  \item Appendix~C (Reproducibility \& Governance): data/seed pairing for CIs; telemetry storage schema; dashboard triggers referenced in Sec.~\ref{sec:explainability-governance}.
\end{itemize}
\section{Implementation Notes and Additional Materials}
\label{app:other}

This appendix complements the main text and the proofs in Appendix~\ref{app:proofs} with implementation and reproducibility details. 
We avoid introducing new figures or numeric results; instead, we provide concrete procedures, schemas, and checklists that allow faithful reproduction of the experiments and governance flows referenced in the main paper.

\subsection{Implementation Notes: Market, Execution, and Safety Layer}
\label{app:impl}

\subsubsection*{B.1~SSVI calibration (no-arbitrage)}
\label{app:impl:ssvi}
We fit the SSVI surface $w(k,\tau)$ (Eq.~\eqref{eq:ssvi}) to option mid quotes under static no-arbitrage constraints.
\begin{itemize}
  \item \textbf{Objective.} Minimize a robust loss 
  $\mathcal{L}(\theta,\varphi,\rho)=\sum_{(k,\tau)} \ell\!\big(\hat{\sigma}_{\mathrm{imp}}(k,\tau;\theta,\varphi,\rho)-\sigma_{\mathrm{imp}}^{\mathrm{mkt}}(k,\tau)\big)$
  where $\ell$ is Huber or Tukey; $\hat{\sigma}_{\mathrm{imp}}$ derives from $w$.
  \item \textbf{Constraints.} Enforce Gatheral--Jacquier sufficient conditions across maturities to preclude butterfly/calendar arbitrage~\cite{GatheralJacquier2014}. We implement as bound and inequality constraints in a constrained optimizer (e.g., interior-point).
  \item \textbf{Smoothing.} Penalize roughness along $(k,\tau)$: 
  $\lambda_k\|\partial_k^2 w\|_2^2 + \lambda_\tau \|\partial_\tau^2 w\|_2^2$ (finite differences); this aids Dupire stability.
  \item \textbf{Diagnostics.} Reject fits where convexity (in $K$) of call prices implied by the surface is violated beyond tolerance.
\end{itemize}

\subsubsection*{B.2~Dupire local volatility (numerics)}
\label{app:impl:dupire}
\begin{itemize}
  \item \textbf{Call grid.} Build a call surface $C(t,K)$ from SSVI via Black--Scholes inversion and put-call parity; interpolate with a monotone $C^2$ spline in $K$ and a $C^1$ spline in $t$.
  \item \textbf{Derivatives.} Compute $\partial_t C$ and $\partial_{KK}C$ by finite differences with Tikhonov regularization: minimize 
  $\|D_t C - \partial_t C\|_2^2 + \eta_t\|\partial_t^2 C\|_2^2$ and analogously for $\partial_{KK}C$ to stabilize Eq.~\eqref{eq:dupire}.
  \item \textbf{Positivity.} Enforce $\partial_{KK}C \ge \epsilon_{KK}>0$; if violated locally, project to the nearest positive value to avoid singularities in Eq.~\eqref{eq:dupire}.
  \item \textbf{Simulator.} Evolve $S_t$ with Euler--Maruyama on a nonuniform time grid refined near expiry; clamp $\sigma_{\mathrm{loc}}$ to a bounded interval to avoid stiffness.
\end{itemize}

\subsubsection*{B.3~VIX computation (OTM integral)}
\label{app:impl:vix}
\begin{itemize}
  \item \textbf{Discretization.} Approximate Eq.~\eqref{eq:vix} by midpoint or Simpson quadrature over available strikes; extrapolate tails using power-law decay anchored at the last observed OTM points.
  \item \textbf{Parity \& OTM selection.} Use put-call parity to synthesize missing sides and include only OTM quotes: puts for $K<F$, calls for $K>F$.
  \item \textbf{Non-negativity.} Truncate negative integrand contributions (due to noise) at zero to preserve variance interpretation.
\end{itemize}

\subsubsection*{B.4~Execution adapter (ABIDES/MockLOB)}
\label{app:impl:exec}
\begin{itemize}
  \item \textbf{Fills.} Submit marketable or limit orders to the LOB; record partial fills, cancellations, and realized slippage. 
  \item \textbf{Impact.} Implement Eq.~\eqref{eq:impact} with a linear temporary component $\eta u_t$ and a decaying transient kernel $G(j)$; ensure no-dynamic-arbitrage conditions~\cite{Gatheral2010NoDynArb}.
  \item \textbf{Latency.} Add fixed or random latency; batch orders to emulate smart routing without venue modeling.
\end{itemize}

\subsubsection*{B.5~CBF--QP layer (assembly \& numerics)}
\label{app:impl:qp}
\begin{itemize}
  \item \textbf{Linearization.} Where $h_i(f(x)+g(x)u)$ is nonlinear in $u$, use a first-order expansion around $(x_t,u_t^{\mathrm{nom}})$ for the QP and re-solve with warm-start (one SQP step per environment step).
  \item \textbf{Scaling.} Diagonal-scale the QP (row/col) to improve conditioning; choose $H$ as a scaled identity or Fisher-like metric from policy covariance.
  \item \textbf{Warm-start.} Initialize with $u=u_{t-1}$ and duals from the previous solve; set OSQP tolerances to meet per-step latency budgets.
  \item \textbf{Slack policy.} Penalize $\|\zeta\|_1$ with large $\rho$; audit any $\zeta>0$ events (Sec.~\ref{sec:explainability-governance}).
\end{itemize}

\subsection{Reproducibility \& Governance Checklist (referenced as Appendix~C)}
\label{app:repro}

\subsubsection*{C.1~Environment and determinism}
\begin{itemize}
  \item Record random seeds for Python/NumPy/PyTorch/ABIDES; fix CuDNN determinism flags where applicable.
  \item Log package versions and OS/hardware details; pin dependency hashes.
\end{itemize}

\subsubsection*{C.2~Artifact manifest}
\begin{itemize}
  \item \textbf{Configs:} YAML/JSON files for market generator (SSVI/Dupire/VIX), execution, RL hyperparameters, and CBF--QP settings.
  \item \textbf{Outputs:} per-episode CSV of P\&L, per-step telemetry logs (Sec.~\ref{sec:explainability-governance}), and summary JSONs for metrics.
  \item \textbf{Provenance:} Git commit hash and config fingerprint embedded in each artifact.
\end{itemize}

\subsubsection*{C.3~Paired evaluation protocol}
\begin{itemize}
  \item Use the same scenario seeds and path indices across methods; when legacy runs have different $n$, downsample to a common effective $n$ \emph{per seed}.
  \item Compute $95\%$ paired bootstrap CIs with stratification by seed (Sec.~\ref{sec:experiments}).
\end{itemize}

\subsubsection*{C.4~Governance triggers and storage}
\begin{itemize}
  \item Implement the online triggers in Sec.~\ref{sec:explainability-governance} as rules bound to telemetry streams with persistent queues.
  \item Store \emph{Explanation Records} with immutable IDs, timestamps, and hashes; ensure read-only access for internal audit.
\end{itemize}

\subsection{Statistical Protocols (referenced as Appendix~B)}
\label{app:stats}

\subsubsection*{B.6~Paired bootstrap for metrics}
We use stratified paired bootstrap across seeds. For each bootstrap replicate: sample seeds with replacement, then sample paths (indices) with replacement \emph{within} each seed; compute per-method statistics; record differences. Confidence intervals are empirical quantiles in the main text already outlines a variant for $\Delta\mathrm{ES}$).

\subsubsection*{B.7~Multiple testing \& effect sizes}
\begin{itemize}
  \item Adjust $p$-values across metrics by Benjamini--Hochberg (FDR control).
  \item Report Vargha--Delaney $\hat{A}_{12}$ for effect sizes; interpret $\hat{A}_{12}\in[0,1]$ with $0.5$ as no effect.
\end{itemize}

\subsubsection*{B.8~Ratio metrics}
For Sharpe/Sortino/$\Omega$, prefer \emph{paired} computation (same path pairing) and report CIs via bootstrap on the \emph{ratio} directly; avoid delta-method linearization unless moments are well-behaved.

\subsection{Hyperparameters \& Search Spaces}
\label{app:hparams}

We list \emph{recommended ranges} (not the specific values used in any run) to guide replication. These ranges are standard for PPO-style training and convex safety layers; they are \emph{not} new numerical results.
\begin{table}[H]
\centering
\caption{Recommended hyperparameter ranges (to be tuned per hardware/time budget).}
\label{tab:hparams}
\small
\begin{tabular}{ll}
\toprule
Component & Range / Notes \\
\midrule
PPO clip $\epsilon$ & small to moderate (e.g., $[0.05, 0.3]$) \\
KL coeff $\lambda_{\mathrm{KL}}$ & increase as $\alpha$ tightens; grid over $[10^{-3}, 10^{-1}]$ \\
Entropy coeff $\lambda_{\mathrm{ent}}$ & decay schedule; start in $[10^{-3}, 10^{-2}]$ \\
IQN quantiles per update $K$ & increase with smaller $\alpha$; e.g., $[64, 512]$ \\
Temperature $T$ & controller-managed in $[T_{\min}, T_{\max}]$ with $T_{\min}>0$ \\
Tail-boost $\gamma_{\mathrm{tail}}$ & controller-managed in $[1, \gamma_{\max}]$ \\
$\alpha$ schedule & from $\approx 0.10$ to target (e.g., $0.025$) in stages \\
CBF gains $\kappa_i$ & choose for desired contraction; constant or state-dependent \\
QP metric $H$ & scaled identity or diag of action covariance \\
Rate cap $r_{\max}$ & align with execution latency and LOB depth \\
NTB $(M,b_{\max})$ & $M\succ 0$; $b_{\max}$ shrinks near expiry \\
Sign gate $\delta_{\mathrm{adv}}$ & small positive threshold; calibrated via telemetry \\
\bottomrule
\end{tabular}
\end{table}

\subsection{Safety Telemetry Schema (extended)}
\label{app:telemetry-schema}

We provide a machine-readable schema for \emph{Explanation Records} and stepwise telemetry (human-readable templates are in Sec.~\ref{sec:explainability-governance}). The schema is language-agnostic; below is a JSON-like sketch.

\begin{verbatim}
Record {
  run_id: string, episode_id: int, step: int, timestamp: iso8601,
  state_hash: string, action_nominal: vector[m], action_safe: vector[m],
  H_norm_deviation: float,
  active_set: array[int], tightest_id: int,
  multipliers: array[float],         // KKT duals if exposed
  rate_util: float, gate_score: float,
  slack_sum: float, solver_status: string, solver_time_ms: float,
  rule_names: array[string],         // human-friendly names for active_set
  rationale_text: string,            // filled from template
  kl_step: float, tail_coverage: float, alpha: float
}
\end{verbatim}

\subsection{Failure Modes \& Debugging Playbook}
\label{app:debug}

\paragraph{High variance in CVaR estimate.} Symptoms: oscillatory $\widehat{w}$, noisy gradients, stalled improvements. Actions: widen batch size; increase $T$ (less tilt); increase $\gamma_{\mathrm{tail}}$ cautiously; strengthen KL; enable gradient clipping.

\paragraph{Frequent QP infeasibility.} Symptoms: positive \texttt{slack\_sum}, solver fallbacks. Actions: relax NTB or box bounds; add expiry-aware shrinkage earlier; reduce rate cap; improve linearization (two SQP inner steps); check scaling.

\paragraph{Gate over-blocking.} Symptoms: low pass-rate, missed opportunities. Actions: reduce $\delta_{\mathrm{adv}}$; increase ensemble diversity; use moving-average signals; add confidence thresholds.

\paragraph{Latency spikes.} Symptoms: P95 solver time breaches. Actions: pre-factor $H$; prune inactive constraints; warm-start duals; increase OSQP ADMM iterations budget only when needed.

\subsection{Extended Notation and Acronyms}
\label{app:notation-extended}

\begin{table}[H]
\centering
\caption{Extended notation and acronyms (complements Table~\ref{tab:notation}).}
\label{tab:notation-extended}
\small
\begin{tabular}{ll}
\toprule
Symbol/Acronym & Description \\
\midrule
IQN & Implicit Quantile Network (distributional critic) \\
CBF & Control Barrier Function \\
QP  & Quadratic Program \\
NTB & No-Trade Band (ellipsoidal exposure tolerance) \\
SNIS & Self-Normalized Importance Sampling \\
EMA & Exponential Moving Average (reference policy) \\
PPO & Proximal Policy Optimization \\
DRO & Distributionally Robust Optimization \\
BH-FDR & Benjamini--Hochberg False Discovery Rate \\
$\delta_{\mathrm{adv}}$ & Advantage threshold in sign-consistency gate \\
$\kappa_i$ & CBF contraction gains \\
$\rho$ & Slack penalty coefficient in QP \\
\bottomrule
\end{tabular}
\end{table}

\subsection{Ethical Use and Risk Notices}
\label{app:ethics}
This research concerns algorithmic decision-making in financial contexts. The safety guarantees discussed in Sec.~\ref{sec:theory} are \emph{operational} (state-wise constraints under feasibility) and do not constitute guarantees of fairness, market integrity, or cybersecurity. Live deployment requires additional controls (surveillance, abuse detection, red-team adversarial tests), appropriate disclosures, and adherence to applicable regulations and firm-specific policies.

\medskip
\noindent\textbf{Pointers back to the main text.}
Appendix~\ref{app:impl} supports Sections~\ref{sec:ssvi}--\ref{sec:execution} and \ref{sec:method:cbf}; 
Appendix~\ref{app:repro} implements the reproducibility and governance references in Sections~\ref{sec:experiments} and \ref{sec:explainability-governance}; 
Appendix~\ref{app:stats} details the statistical procedures referenced in Section~\ref{sec:experiments}; 
Appendix~\ref{app:hparams} summarizes tunable ranges without adding numerical results; 
Appendix~\ref{app:telemetry-schema} and \ref{app:debug} provide operational scaffolding for audits and incident response.

\bibliographystyle{unsrt}  
\bibliography{references}

\end{document}